\documentclass[11pt, notitlepage]{article}   
\usepackage[english]{babel}

\usepackage[letterpaper,top=2.5cm,bottom=2.5cm,left=2.5cm,right=2.5cm,marginparwidth=1.75cm]{geometry}

\usepackage{amsfonts,amsmath,amssymb,amscd,amsthm,euscript,graphicx,latexsym,mathrsfs,verbatim,setspace} 

\usepackage{authblk}

\theoremstyle{plain}
\newtheorem{thm}{Theorem}[section]
\newtheorem{lem}[thm]{Lemma}

\newtheorem{cor}[thm]{Corollary}

\newtheorem{conj}[thm]{Conjecture}

\theoremstyle{definition}
\newtheorem{defin}{Definition}[section]
\usepackage{mathtools}
\usepackage{amsthm}
\usepackage{amssymb}
\usepackage{amsfonts}
\usepackage{graphicx}
\usepackage{ytableau}
\usepackage{hyperref}
\usepackage{cleveref}
\usepackage[utf8]{inputenc}
\usepackage{todonotes}
\usepackage{tikz}
\usepackage{thmtools}
\usepackage{thm-restate}
\usepackage{comment}

\newcommand{\opt}{\operatorname{opt}}
\newcommand{\optag}{\textnormal{opt}^{\textnormal{nf}}}

\title{Worst-case Error Bounds for Online Learning of Smooth Functions}
\author{Weian (Andrew) Xie}

\begin{document}
\maketitle

\begin{abstract}
Online learning is a model of machine learning where the learner is trained on sequential feedback. We investigate worst-case error for the online learning of real functions that have certain smoothness constraints. Suppose that $\mathcal{F}_q$ is the class of all absolutely continuous functions $f: [0, 1] \rightarrow \mathbb{R}$ such that $\|f'\|_q \le 1$, and $\opt_p(\mathcal{F}_q)$ is the best possible upper bound on the sum of the $p^{\text{th}}$ powers of absolute prediction errors for any number of trials guaranteed by any learner. We show that for any $\delta, \epsilon \in (0, 1)$, $\opt_{1+\delta} (\mathcal{F}_{1+\epsilon}) = O(\min(\delta, \epsilon)^{-1})$. Combined with the previous results of Kimber and Long (1995) and Geneson and Zhou (2023), we achieve a complete characterization of the values of $p, q \ge 1$ that result in $\opt_p(\mathcal{F}_q)$ being finite, a problem open for nearly 30 years.

We study the learning scenarios of smooth functions that also belong to certain special families of functions, such as polynomials. We prove a conjecture by Geneson and Zhou (2023) that it is not any easier to learn a polynomial in $\mathcal{F}_q$ than it is to learn any general function in $\mathcal{F}_q$.  We also define a noisy model for the online learning of smooth functions, where the learner may receive incorrect feedback up to $\eta \ge 1$ times, denoting the worst-case error bound as $\optag_{p, \eta} (\mathcal{F}_q)$. We prove that $\optag_{p, \eta} (\mathcal{F}_q)$ is finite if and only if $\opt_p(\mathcal{F}_q)$ is. Moreover, we prove for all $p, q \ge 2$ and $\eta \ge 1$ that $\optag_{p, \eta} (\mathcal{F}_q) = \Theta (\eta)$.

\emph{Keywords:} online learning, smooth functions, noisy labels, single-variable functions, worst-case error

\end{abstract}

\newpage

\section{Introduction}\label{one}

Imagine a situation where a learner makes daily predictions for stock price values, given relevant inputs such as industry performance, market-wide economic trends, or recent company performance. The next day, the learner obtains some form of feedback on its previous prediction—receiving the actual price value, for example. The learner then uses its increased knowledge of the price value at certain inputs to generate better-informed future predictions—on the basis that similar inputs yield similar price values. The question of how fast the learner can use its accumulated information to generate better predictions in the worst case scenario arises naturally, and is the original motivation for the model of online learning previously studied in \cite{angluin, geneson, littlestone, long, kl, mycielski}. 

Worst-case error bounds for online learning were first studied on functions over a discrete domain and with output values among a finite set, in \cite{angluin, littlestone, mycielski}. The model over real-valued single-variable smooth functions $f:[0, 1] \rightarrow \mathbb{R}$ that we investigate was first introduced in \cite{kl}, later studied in \cite{long} and most recently studied \cite{geneson}. The last paper also extended the problem to multi-variable smooth functions $f: [0, 1]^d \rightarrow \mathbb{R}$. 

As our problem investigates how much accuracy the learner can guarantee in the worst-case scenario, we naturally represent the learning process as a game between the learner and an adversary, the latter of which is trying to force as much error as possible, with both players playing optimally. In the discrete model, different forms of feedback have been studied, including the standard model, where the adversary tells the learner the precise function output at the queried input; the bandit model, where the adversary only tells the learner YES or NO based on whether the learner's answer is correct; and a model where the adversary is allowed to give incorrect reinforcement up to $\eta$ times, for some $\eta \ge 1$. The smooth function problem has so far only been studied in the context of the standard model, and all results so far pertain to this model. Indeed, the bandit model would not be interesting to study in the context of smooth functions, as the adversary can guarantee infinite error each time—through answering NO each time and simply setting the function $f \equiv C$ for some $C$ sufficiently far away from all of the learner's guesses. However, a noisy model is interesting to study. For the noisy model in the smooth function setting, the adversary may provide erroneous output values to the learner.

\subsection{Definitions} 
In the standard model of online learning, an algorithm $A$ tries to learn a function from a given class of functions $\mathcal{F}$. The class $\mathcal{F}$ consists of functions $f: S \rightarrow \mathbb{R}$ for some fixed input set $S$. In the $t^{\text{th}}$ trial, the algorithm is repeatedly given an input $x_t$ within $S$ and queried on the value of $f(x_t)$, whereupon it outputs a prediction $\hat y_t$. Afterwards, the true value of $f(x_t)$ is revealed to $A$, and the raw error of the round $e_t = |\hat y_t - f(x_t)|$ is recorded. The total error function is calculated as the sum of the $p^{\text{th}}$ powers of the raw errors, for some parameter $p \ge 1$. Specifically, as per the notation of \cite{geneson}, for a fixed learning algorithm $A$ operating on a finite sequence of inputs $\sigma = (x_0, x_1, \ldots, x_m) \in S^{m+1}$, and function $f \in \mathcal{F}$, this sum is \[\mathscr L_p(A,f,\sigma) = \sum_{t=1}^m e_t^p = \sum_{t=1}^m|\hat y_t-f(x_t)|^p.\] Subsequently, we define the worst-case scenario learning error for a fixed algorithm $A$ over a class $\mathcal{F}$, as \[\mathscr L_p(A,\mathcal F)=\displaystyle\sup_{f \in \mathcal F,\sigma \in \cup_{m \in \mathbb Z^+}S^m}\mathscr L_p(A,f,\sigma).\] We then define the worst-case error for the best possible learning algorithm: \[\opt_p(\mathcal F)=\displaystyle\inf_A \mathscr L_p(A,\mathcal F).\] 

In \cite{kl}, \cite{long}, and \cite{geneson}, the family $\mathcal{F}$ studied was the class of absolutely continuous single-variable functions $f : [0, 1] \rightarrow \mathbb{R}$ that satisfy certain smoothness constraints, in the form of their derivatives having bounded norms. Namely, for all $q \ge 1$, define $\mathcal{F}_q$ to be the class of absolutely continuous functions $f: [0, 1] \rightarrow \mathbb{R}$ such that $\int_0^1 |f'(x)|^q \le 1$. Naturally, an extension of this definition is $\mathcal{F}_{\infty}$, denoting the class of absolutely continuous functions $f: [0, 1] \rightarrow \mathbb{R}$ such that $\sup_{x \in (0, 1)} |f'(x)| \le 1$. As noted in \cite{geneson} and \cite{long}, $\mathcal{F}_{\infty}$ contains precisely the functions $f: [0, 1] \rightarrow \mathbb{R}$ that satisfy $|f(x)-f(y)| < |x-y|$ for all $x, y \in [0, 1]$. The paper \cite{geneson} notes that for any $q \ge 1$, the range of any function $f \in \mathcal{F}_q$ is at most $1$. Furthermore, \cite{geneson} also notes that $\mathcal{F}_{\infty} \subseteq \mathcal{F}_{q} \subseteq \mathcal{F}_{r}$ for all $1 \le r \le q$ by Jensen's Inequality, from which it follows that $\opt_{p}({\mathcal{F}_{\infty}}) \le \opt_{p}({\mathcal{F}_{q}}) \le \opt_{p}({\mathcal{F}_{r}})$ for all $1 \le r \le q$. 

We also carry over some notation from \cite{kl}, \cite{long}, and \cite{geneson}. Let the $q$-action of a function $f : [0, 1] \rightarrow \mathbb{R}$, denoted by $J_q[f]$, be defined as \[J_q[f] = \int_{0}^1 |f'(x)|^q \text{d}x. \] As such, $\mathcal{F}_q$ is the set of functions whose $q$-action is less than or equal to $1$. 
Furthermore, given a set $S = \{(u_i, v_i): 1 \le i \le m\}$, where each $(u_i, v_i) \in [0, 1] \times \mathbb{R}$ such that $u_i < u_{i+1}$ for all $1 \le i \le m-1$, define $f_S : [0, 1] \rightarrow \mathbb{R}$ such that $f_{\emptyset} \equiv 0$ and \[f_S(x) = \begin{cases}  v_1 & x \le u_1 \\ v_i+\frac{(x-u_i)(v_{i+1}-v_i)}{u_{i+1}-u_i} & x \in (u_i,u_{i+1}] \\ v_m & x>u_m \end{cases}\] if $|S|$ is nonzero. Therefore, graphically, $f_S$ is a continuous piecewise function composed of various line segments.

We now state two facts regarding $f_S$ that will come into use later. Given a set $S = \{(u_i, v_i): 1 \le i \le m\}$, with $(u_i, v_i) \in [0, 1] \times \mathbb{R}$ for each $1 \le i \le m$ and $u_1 < \ldots < u_m$, one useful fact about $f_S$ is that \[J_1[f_S] = \sum_{i=1}^{m-1} \left( \int_{u_i}^{u_{i+1}} \left|\frac{v_{i+1} - v_i}{u_{i+1} - u_i} \right| \text{d}x \right) = \sum_{i=1}^{m-1} |v_{i+1} - v_i|. \] Another useful fact about $f_S$ is that it has the minimum $q$-action out of all functions $f$ passing through all points in $S$:

\begin{lem}[\cite{kl}, \cite{geneson}]\label{linintmin}
    Let $S = \{(u_1,v_1),\ldots,(u_m,v_m) \}$ be a set of $m$ points with $(u_i, v_i) \in [0, 1] \times \mathbb{R}$ for each $i$, such that $u_1 < \ldots < u_m$. Then, for any $q \ge 1$ and any absolutely continuous function $f: [0,1] \to \mathbb R$ with $f(u_i)=v_i$ for $1 \le i \le m$, we have $J_q[f] \ge J_q[f_S]$. 
\end{lem}

As per \cite{kl}, we define the learning algorithm LININT using $f_S$. In particular, on trial $0$, LININT's prediction is $\hat y_0 = \text{LININT}(\emptyset, x_1) = 0$, and in any subsequent trial $t > 0$, its prediction is \[\hat y_t = \text{LININT($(x_0, f(x_0)), \ldots, (x_{t-1}, f(x_{t-1})), x_t$)} = f_{\{(x_0, f(x_0)), \ldots, (x_{t-1}, f(x_{t-1}))\}}(x_t).\] Intuitively, this means that LININT makes a prediction either based on linear interpolation on the two closest surrounding points (one on each side of the input) which the algorithm knows the true value of, or simply based on its closest neighbor, if the requested input is to the left or to the right of all known points.

\subsection{Smooth Functions in the Standard Model}

The problem of determining the value of $\opt_p(\mathcal{F}_q)$ across all $p, q\ge 1$ has been extensively explored in \cite{kl}, \cite{long}, and \cite{geneson}. Kimber and Long \cite{kl} proved that if we set $q=1$, then for any $p \ge 1$, $\opt_p(\mathcal{F}_1) = \infty$. That is, no matter what parameter $p$ is chosen, the learner can never guarantee finite error when learning a function from $\mathcal{F}_1$. Furthermore, the same paper \cite{kl} also established that $\opt_1({\mathcal{F}_{\infty}}) = \infty$, from which it follows that for any $q \ge 1$, we have $\opt_1({\mathcal{F}_q}) = \infty$ as well. On the other hand, they proved that $\opt_p({\mathcal{F}_{q}}) =1$ for all $p, q \ge 2$, from which it follows that $\opt_p(\mathcal{F}_{\infty})=1$ for all $p \ge 2$ as well. This result was also proved in \cite{FM} with a different learning algorithm. 

The paper \cite{kl} also established that $\opt_{1+\epsilon}(\mathcal{F}_{q}) = O(\epsilon^{-1})$ for all $\epsilon \in (0,1)$ and all $q \ge 2$. Geneson and Zhou \cite{geneson} improved this bound and established a lower bound differing by a constant factor, proving that $\opt_{1+\epsilon}(\mathcal{F}_{\infty}) = \Theta(\epsilon^{-\frac{1}{2}})$ and $\opt_{1+\epsilon}(\mathcal{F}_{q}) = \Theta(\epsilon^{-\frac{1}{2}})$ for all $\epsilon \in (0,1)$ and all $q \ge 2$. They also established that $\opt_{2}(\mathcal{F}_{1+\epsilon}) = \Theta(\epsilon^{-1})$ for any $\epsilon \in (0, 1)$. From this, the upper bound $\opt_p(\mathcal{F}_{1+\epsilon}) = O(\epsilon^{-1})$ for all $p \ge 2$ and $\epsilon \in (0, 1)$ follows. Furthermore, for any $q > 1$, they established that for sufficiently large $p$, the learner can guarantee an error of at most $1$. Specifically, for any $q>1$ and $p \ge 2 + \frac{1}{q-1}$, $\opt_p(\mathcal{F}_q) = 1$. 

As such, upper bounds for $\opt_p(\mathcal{F}_q)$ have been established for all $p, q > 1$ whenever at least one of $p \ge 2$ and $q \ge 2$ holds. However, in the entirety of previous literature, no upper bounds for any instance of $p, q \in (1,2)$ have been proved. In fact, there was no proof of finiteness for $\opt_p(\mathcal{F}_q)$ for any choice of $p, q \in (1,2)$. In this paper, we establish an upper bound on $\opt_p(\mathcal{F}_q)$ for all values of $p, q \in (1,2)$. 

\begin{thm}\label{finitebound}
    For $\delta, \epsilon \in (0, 1)$, we have $\opt_{1+\delta}(\mathcal{F}_{1+\epsilon}) = O(\min(\delta, \epsilon)^{-1})$.
\end{thm}

As a corollary, we have the following result, which was also a conjecture by Geneson and Zhou \cite{geneson}.

\begin{thm}\label{finite}
    For all $p >1$ and $q > 1$, $\opt_p(\mathcal{F}_q)$ is finite. 
\end{thm} 

Combined with the results from $\cite{kl}$ that $\opt_p(\mathcal{F}_q) = \infty$ whenever either $p=1$ or $q=1$, we now achieve a complete characterization of the values of $(p, q)$ for which $\opt_p(\mathcal{F}_q)$ is finite (i.e. precisely when $p, q >1$), answering a question that has been open since the model of the online learning was first defined for smooth functions in 1995 by Kimber and Long \cite{kl}. 

\subsection{Special Families of Smooth Functions}

Geneson and Zhou \cite{geneson} first explored the online learning of special families of smooth functions with further imposed restrictions. They suggested studying families such as polynomials, sums of exponential functions, sums of trigonometric functions, and piecewise combinations of these functions. Placing extra restrictions on smooth functions is natural as most functions modeling real-life phenomena belong to families that have additional properties beyond merely than the smoothness constraints that were extensively studied. As such, this direction can provide new results that are more specifically applicable to real-life learning scenarios. 

Carrying over the notation from \cite{geneson}, let $\mathcal{P}_q \subseteq \mathcal{F}_q$ denote the family of polynomials $f$
 such that $f \in \mathcal{F}_q$. We prove a conjecture from \cite{geneson} in the affirmative, establishing that given a fixed $q$-action restriction on a smooth function, having the extra restriction of the function being a polynomial does not decrease the learner's error. 

\begin{thm}\label{polynomialconjecture}
For all $p > 0$ and $q \ge 1$, we have $\opt_p(\mathcal{P}_q) = \opt_p(\mathcal{F}_q)$.
\end{thm}

\subsection{Smooth Functions with Noisy Feedback}\label{agnostic}

The online learning of functions with noisy feedback has previously been studied in the discrete setting for classifiers, in \cite{cesa, Auer, Filmus}, and we extend it here to the context of smooth functions as well. This model would be more applicable than the previously studied standard model. Suppose that a certain phenomenon to be learned cannot be perfectly represented by any function within a class, and the goal is to simply find a function from the class that can model the phenomenon as well as possible. As such, these inaccuracies correspond to the lies made by the adversary in our model. 

Suppose that in the same format as in the standard model, an algorithm $A$ attempts to learn a real-valued function $f: [0, 1] \rightarrow \mathbb{R}$ from a certain class $\mathcal{F}$. In the $t^{\text{th}}$ trial, the algorithm is queried on the value of $f$ at an input $x_t$, whereupon it outputs a prediction $\hat y_t$ and the adversary subsequently reveals the actual value of $f(x_t)$. However, for a fixed positive integer $\eta \ge 1$, the adversary is allowed to lie about the value of $f(x_t)$ for up to $\eta$ trials $t$. We define the error function similarly to our definition in the standard case, as the sum of the $p^{\text{th}}$ powers of the raw errors of each trial $e_t = |\hat y_t - f(x_t)|$. However, as the real value of the function at $x_t$, $f(x_t)$, may differ from the adversary's reported value, the learner may possibly not know the exact value of the error function. Regardless, we are interested in how much accuracy the learner can guarantee in the worst-case scenario in the presence of possible lies from the adversary. 

In the standard model of the learning of smooth functions, the error made by the learner on the first trial, $e_0$, is not counted in the error function, as otherwise the adversary can simply set the function $f$ identically equal to some constant $C$ sufficiently far away from the learner's prediction to generate an arbitrarily large error. As such, the learner needs to know the value of the function for at least one input to guarantee finite error on its first prediction. In this noisy version, where the learner is allowed to lie up to $\eta$ times, getting feedback at one input is not sufficient to guarantee finite error on the next prediction for the learner, as the adversary can simply lie. There is no point in studying the worst-case error of the learner when the adversary can always force infinite error for the learner on its first counted trial. As such, it is necessary that we give the learner more initial rounds of feedback. We establish the following result about the number of initial rounds of adversary feedback necessary for the learner to guarantee finite error on its first real prediction. 

\begin{thm}\label{1.5}
    For any integer $\eta \ge 1$, if incorrect feedback can be given up to $\eta$ times, then at least $2\eta+1$ initial rounds must be thrown out for the learner to guarantee finite error on its first prediction that counts toward the error evaluation.    
\end{thm}

Accordingly, let the error function calculate the sum of the $p^{\text{th}}$ powers of the raw errors starting from the $2\eta+2^{\text{nd}}$ trial. In particular, for a fixed learning algorithm $A$ operating on a finite sequence of inputs $\sigma = (x_0, x_1, \ldots, x_m) \in [0, 1]^{m+1}$ and fixed function $f \in \mathcal{F}$, define the error function for the noisy model to be \[\mathscr L_{p, \eta}^{\text{ag}}(A,f,\sigma) = \sum_{t=2\eta+1}^m e_t^p = \sum_{t=2\eta+1}^m|\hat y_t-f(x_t)|^p.\] 

Similar to the definition of the standard model, we denote the worst-case error for a fixed algorithm $A$ over a class $\mathcal{F}$ of functions $f: [0, 1] \rightarrow \mathbb{R}$ as \[\mathscr L_{p, \eta}^{\text{ag}}(A,\mathcal F)=\displaystyle\sup_{f \in \mathcal F,\sigma \in \cup_{m \in \mathbb Z^+}[0,1]^m}\mathscr L_{p, \eta}^{\text{ag}}(A,f,\sigma).\] 

Finally, we define the worst-case error for the best possible learning algorithm over a class $\mathcal{F}$, where the adversary is allowed to lie up to $\eta$ times: \[\opt_{p, \eta}^{\text{ag}}(\mathcal F)=\displaystyle\inf_A \mathscr L_{p, \eta}^{\text{ag}}(A,\mathcal F).\] 

We first establish a characterization for when $\opt_{p, \eta}^{\text{ag}}(\mathcal F_q)$ is finite.

\begin{thm}\label{agnosticupperbound}
    For any integer $\eta \ge 1$, the value of $\optag_{p, \eta} (\mathcal{F}_q)$ is finite if and only if $p > 1$ and $q > 1$. 
\end{thm}

For $p, q \ge 2$ and any $\eta \ge 1$, we show that the noisy worst-case error is precisely on the order of $\eta$.

\begin{thm}\label{agnosticlowerbound}
    For any $\eta \ge 1$, $p, q \ge 2$ we have $\optag_{p, \eta} (\mathcal{F}_q) = \Theta(\eta)$.
\end{thm}

\subsection{Order of Results}\label{results}
In Section \ref{standard}, we work with the standard single-variable model, proving Theorem \ref{finitebound} and its corollary Theorem \ref{finite}. In Section \ref{2}, we study the online learning of smooth polynomials, proving Theorem \ref{polynomialconjecture}. In Section \ref{3}, we focus on the noisy learning scenario, defining its setup and establishing some preliminary bounds on $\optag_{p, \eta} (\mathcal{F}_q)$. In Section \ref{discussion}, we discuss open problems, conjectures, and future directions of research. 


\section{An Upper Bound on $\opt_p(\mathcal{F}_q)$ for $p, q \in (1, 2)$}\label{standard}

In this section, we prove that for all $\delta, \epsilon \in (0, 1)$, we have $\opt_{1+\delta}(\mathcal{F}_{1+\epsilon}) = O(\min(\delta, \epsilon)^{-1})$. To prove this upper bound, we will use the LININT learning algorithm first defined by \cite{kl}, which has previously been used to prove various upper bounds on $\opt_p(\mathcal{F}_q)$ in \cite{kl}, \cite{long}, and  \cite{geneson}. Specifically, we will first prove that $\mathscr L_{1+\epsilon}(\text{LININT},\mathcal F_{1+\epsilon}) = O(\epsilon^{-1})$ for $\epsilon \in (0,1)$. We will then use this to bound $\opt_{1+\delta}(\mathcal{F}_{1+\epsilon})$ for all $\epsilon, \delta \in (0, 1)$. 

The main idea for the first part of our proof that $\mathscr L_{1+\epsilon}(\text{LININT},\mathcal F_{1+\epsilon}) = O(\epsilon^{-1})$ is similar to the main idea of previous proofs of upper bounds on $\mathscr L_p(\text{LININT},\mathcal F_{q})$. Specifically, we will compare changes in $J_{1+\epsilon}[f_S]$ as new points are added to the input-output set $S$ to powers of $|y-f_S(x)|$, the raw absolute errors generated by the corresponding rounds, culminating in Lemma \ref{partone}. This approach is similar to that of Lemma 10 in \cite{kl} and Lemma 2.10 in \cite{geneson}. The main difference is that we will use novel inequality tactics, such as the binomial series expansion, to prove stronger inequalities than the inequalities previously proved. In the second part of our proof, we establish Lemma \ref{twopointsix}, a novel inequality that holds for an arbitrary number of rounds. Our result then follows upon combining Lemmas \ref{partone} and \ref{twopointsix}. 

From this point onwards, we assume without loss of generality that the learning algorithm is never queried on the same input more than once, as the algorithm can always guarantee zero error upon being queried on the same input the second time. First, we prove a modification of the inequality that is Corollary 2.9 from Geneson and Zhou \cite{geneson}. In particular, we specify that $0 < a \le b < 1$ and expand the domain of $x$ that the inequality works for from all $x \notin (a, b)$ to all $x \notin (-a, a)$, at the cost of weakening the inequality by a constant factor. 

\begin{lem}\label{2qoutcor}
For reals $0 < a \le b < 1$ such that $a+b \le 1$, $q \in (1,2)$, and any $|x| \ge a$, we have \[ a\left|\frac{x}{a}+1\right|^q+b\left|\frac{x}{b}-1\right|^q-(a+b) \ge \frac{(q-1)|x|^q}{3}. \]
\end{lem}

\begin{proof}
    Note that for all $x \in (-\infty, -a] \cup [b, \infty)$, the inequality directly follows from Corollary 2.9 from Geneson and Zhou \cite{geneson}. Therefore, we only have to deal with the case where $x \in [a, b)$. For fixed $a$ and $b$, let 
    \begin{align*}
        f(x) &= a\left|\frac{x}{a}+1\right|^q+b\left|\frac{x}{b}-1\right|^q-(a+b) - \frac{(q-1)|x|^q}{3} \\
        &= a \left(\frac{x}{a}+1 \right)^q + b \left(1- \frac{x}{b}\right)^q -(a+b) - \frac{(q-1)x^q}{3}
    \end{align*}
    when $x \in [a, b]$. We prove that $f'(x) > 0$ for all $x \ge a$, so that it suffices to check that $f(a) \ge 0$. Indeed,
    \begin{align*}
        f'(x) &= q \left(\frac{x}{a} + 1\right)^{q-1}-q \left(1 - \frac{x}{b}\right)^{q-1} - \frac{q(q-1) x^{q-1}}{3} \\
        &\ge q \cdot 2^{q-1} - q - \frac{q(q-1)}{3} = q \left(2^{q-1} - 1 - \frac{q-1}{3} \right) > 0
    \end{align*} for $q \in (1, 2)$, when $a \le x \le b$. Consider the function $g(x) = b \left(1 - \frac{x}{b} \right)^q$. Note that $g'(x) = -q \left(1 - \frac{x}{b}\right)^{q-1} \ge -q$ when $x \in [0, b]$. As $g(0) = b$, we have $g(x) \ge b - qx$ for all $x \in [0, b]$. As such, $b \left(1 - \frac{a}{b} \right)^q = g(a) \ge b-qa$. Now, we have 

    \begin{align*}
        f(a) &= a \cdot 2^q + b \left(1 - \frac{a}{b} \right)^q - (a+b) - \frac{(q-1) a^q}{3} \ge a \cdot 2^q + (b - qa) - (a+b) - \frac{(q-1)a^q}{3} \\
        &\ge a \cdot 2^q + (b - qa) - (a+b) - \frac{(q-1)a}{3} = a(2^q - q - 1) - \frac{(q-1)a}{3} \\
        &\ge a\left(\frac{q-1}{3} \right) - \frac{(q-1)a}{3} = 0.
    \end{align*}
    Hence, we have that $f(x) \ge 0$ for all $x \in [a, b)$, as desired. 
\end{proof}

We now prove a modification of another inequality, Lemma 2.7, from \cite{geneson}. Specifically, under the assumption that $a \le b$, we will improve the factor of $a+b$ to a factor of $a$ on the denominator of the right hand side, at the cost of a constant factor, which is a significant improvement from the bound in \cite{geneson} when $a$ is significantly smaller than $b$. However, we will restrict the domain from $x \in (-a, b)$ to $x \in (-a, a)$, given that we have already dealt with the case where $x \in [a, b)$ in Lemma \ref{2qoutcor}. 

\begin{lem}\label{2qin}
For reals $0 < a \le b$, $q \in (1,2)$, and $x \in (-a,a)$, we have \[ a\left(1+\frac{x}{a}\right)^q+b\left(1-\frac{x}{b}\right)^q-(a+b) \ge \frac{q(q-1)}{3a} \cdot x^2. \]
\end{lem}

\begin{proof}
    We use the generalized binomial series expansion of $\left(1+\frac{x}{a}\right)^q$ and $\left(1-\frac{x}{b}\right)^q$ to rewrite the expression on the left hand side, which we will call $\Delta$ for convenience. Doing this is valid as by assumption $\left| \frac{x}{a} \right| < 1$ and $\left| \frac{x}{b} \right| < 1$, so the series expansions converge correctly. Then,
    \begin{align*}
        \Delta &= a \left(\sum_{k=0}^{\infty} \binom{q}{k} \left( \frac{x}{a} \right)^k \right) + b \left(\sum_{k=0}^{\infty} \binom{q}{k} \left( -\frac{x}{b} \right)^k \right) - (a+b) \\
        &= a \left(\sum_{k=1}^{\infty} \binom{q}{k} \left( \frac{x}{a} \right)^k \right) + b \left(\sum_{k=1}^{\infty} \binom{q}{k} \left( -\frac{x}{b} \right)^k \right) \\
        &= qx + a \left(\sum_{k=2}^{\infty} \binom{q}{k} \left( \frac{x}{a} \right)^k \right) -qx + b \left(\sum_{k=2}^{\infty} \binom{q}{k} \left( -\frac{x}{b} \right)^k \right) \\
        &=  a \left(\sum_{k=2}^{\infty} \binom{q}{k} \left( \frac{x}{a} \right)^k \right) + b \left(\sum_{k=2}^{\infty} \binom{q}{k} \left( -\frac{x}{b} \right)^k \right).
    \end{align*}

    We claim that $\left(\sum_{k=2}^{\infty} \binom{q}{k} \left( \frac{x}{a} \right)^k \right) \ge \frac{q(q-1)}{3} \left(\frac{x}{a} \right)^2$. The inequality is trivial if $\frac{x}{a} = 0$. If $\frac{x}{a} < 0$, note that every term of the sum is nonnegative, because $\binom{q}{k}$ is negative precisely when $k \ge 2$ is odd, from the condition that $q \in (1, 2)$, and $\left(\frac{x}{a} \right)^k$ is also negative precisely when $k$ is odd. Therefore, we have $\left(\sum_{k=2}^{\infty} \binom{q}{k} \left( \frac{x}{a} \right)^k \right) \ge \frac{q(q-1)}{2} \left(\frac{x}{a} \right)^2$ clearly. If $\frac{x}{a} > 0$, note that $\binom{q}{k} \left( \frac{x}{a} \right)^k > 0$ when $k \ge 2$ is even, and $\binom{q}{k} \left( \frac{x}{a} \right)^k < 0$ when $k \ge 2$ is odd. Furthermore, for any $k \ge 2$, we have $\left| \binom{q}{k} \left( \frac{x}{a} \right)^{k} \right| > \left| \binom{q}{k+1} \left( \frac{x}{a} \right)^{k+1} \right|$ because $\left| \binom{q}{k+1} \right| = \left|\binom{q}{k} \cdot \frac{q-k}{k+1} \right| = \left|\binom{q}{k} \right| \cdot \left|\frac{k-q}{k+1} \right| < \left|\binom{q}{k} \right|$ when $q \in (1, 2)$ and $\left|\left( \frac{x}{a} \right)^{k+1} \right| < \left|\left( \frac{x}{a} \right)^k \right|$. As such, because $\left(\sum_{k=2}^{\infty} \binom{q}{k} \left( \frac{x}{a} \right)^k \right)$ is an alternating series with the magnitude of summands decreasing, the sum is bounded below by \[\binom{q}{2} \left( \frac{x}{a} \right)^2 + \binom{q}{3} \left( \frac{x}{a} \right)^3 > \left(\binom{q}{2} + \binom{q}{3} \right) \left(\frac{x}{a} \right)^2 \ge \frac{2}{3} \binom{q}{2} \left(\frac{x}{a} \right)^2 \ge \frac{q(q-1)}{3} \left(\frac{x}{a} \right)^2\] when $q \in (1, 2)$, as claimed.

    Similarly, we can conclude that $\left(\sum_{k=2}^{\infty} \binom{q}{k} \left( -\frac{x}{b} \right)^k \right) \ge \frac{q(q-1)}{3} \left(- \frac{x}{b} \right)^2$, which is nonnnegative. Using these inequalities, we get the lower bound on our desired expression:
    \[ \Delta \ge a \left( \frac{q(q-1)}{3} \left(\frac{x}{a} \right)^2 \right) + b\cdot 0\ge \frac{q(q-1)}{3 a} \cdot x^2.\]
\end{proof}

Combining the above two inequalities, we make the following key claim, which is essentially a strengthened and more specific version of Lemma 2.10 in \cite{geneson}.

\begin{lem}\label{partone}
    For a fixed $q \in (1, 2)$, a nonempty set $S = \{(u_i, v_i) : 1 \le i \le k\}$ of points in $[0, 1] \times \mathbb{R}$ with $u_{i} < u_{i+1}$ for each $1 \le i \le k-1$, and another point $(x, y) \in [0, 1] \times \mathbb{R}$ such that $x \neq u_i$ for any $i$, we must have either \[J_q[f_{S \cup \{(x, y)\}}] - J_q[f_S] \ge \frac{q-1}{3} \cdot |y-f_S(x)|^q\] or \[J_q[f_{S \cup \{(x, y)\}}] - J_q[f_S] \ge \frac{(q-1)}{3 |m|^{2-q} \cdot d} \cdot (y-f_S(x))^2,\] where we let $d = \min_i |x-u_i|$ and $m = f_S'(x)$, the slope of the linear interpolation function at $x$. 
\end{lem}

\begin{proof}
    First, if $x < u_1$, as established in the proof of Lemma 2.10 in the paper \cite{geneson}, \[J_q[f_{S \cup \{(x, y)\}}] - J_q[f_S] = (u_1-x) \left|\frac{v_1-y}{u_1-x} \right|^q. \] Because $|u_1-x| \le 1$, the above expression is at least $|v_1-y|^q = |y-f_S(x)|^q \ge (q-1)|y-f_S(x)|^q$, hence satisfying the first inequality. The case of $x > u_k$ is similar. 

    Meanwhile, if $u_i < x < u_{i+1}$ for some integer $1 \le i \le k-1$, substituting $a=x-u_i,b=u_{i+1}-x,c=y-f_S(x)$, and $m= f'_S(x) = \frac{v_{i+1}-v_i}{u_{i+1}-u_i}=\frac{f_S(x)-v_i}{a}=\frac{v_{i+1}-f_S(x)}{b}$, we have \[J_q\left[f_{S \cup \{(x,y)\}}\right]-J_q[f_S] = |m|^q\left(a\left|1+\frac{c}{ma}\right|^q+b\left|1-\frac{c}{mb}\right|^q-(a+b)\right),\] as established in the proof of Lemma 2.10 in \cite{geneson}. Without loss of generality, assume that $a = |x -u_{i} | \le | x - u_{i+1}| = b$, as the other case follows from symmetry. Note that $a = d = \min_i |x - u_i|$. If $\frac{c}{m} \notin (-a, a)$, applying Lemma \ref{2qoutcor}, we have \[J_q\left[f_{S \cup \{(x,y)\}}\right]-J_q[f_S] \ge |m|^q \left(\frac{q-1}{3} \left| \frac{c}{m} \right|^q \right) =  \frac{q-1}{3} \cdot |c|^q = \frac{q-1}{3} \cdot |y-f_S(x)|^q, \] satisfying the first inequality. On the other hand, if $\frac{c}{m} \in (-a, a)$, applying Lemma \ref{2qin}, we have 
    \begin{align*}
        J_q\left[f_{S \cup \{(x,y)\}}\right]-J_q[f_S] &\ge |m|^q \left(\frac{q(q-1)}{3 a} \cdot \left| \frac{c}{m} \right|^2 \right) \\ 
        &\ge |m|^q \left(\frac{(q-1)}{3 d} \cdot \left| \frac{c}{m} \right|^2 \right) = \frac{q-1}{3|m|^{2-q} \cdot d} \cdot (y-f_S(x))^2.
    \end{align*} As a result, in the case where $\frac{c}{m} \in (-a, a)$, the second inequality is satisfied. Hence, in any case, either the first or the second inequality needs to hold. 
\end{proof}

Now, we proceed to the second part of our proof, where we prove another bound, Lemma \ref{twopointsix}, which we will combine at the end with Lemma \ref{partone} using Hölder's Inequality to prove our desired result. We first establish an inequality, which will be used in proving Lemma \ref{twopointsix}. 

\begin{lem}\label{twovariable}
    For all reals $p>1$ and $x \ge 2$, we have $$x^p - (x-1)^{p-1} \cdot x \ge p-1.$$
\end{lem}

\begin{proof}
    We consider two cases, dependent on whether $p \ge 2$ or $1 < p < 2$. If $p \ge 2$, then for all $x \ge 2$,
    \begin{align*}
    x^p-(x-1)^{p-1} \cdot x &= x \left(x^{p-1}-(x-1)^{p-1} \right) \ge x(2^{p-1}-1^{p-1}) \\
    &\ge 2 \cdot (2^{p-1}-1^{p-1}) = 2^p - 2 \ge p-1,
    \end{align*}
    where the second step follows from the function $f(x)=x^{p-1} - (x-1)^{p-1}$ being non-decreasing for $x \ge 2$, which can be verified by differentiation. If $1 < p < 2$, then for all $x \ge 2$, we have
    \begin{align*}
        x^p - (x-1)^{p-1} \cdot x &= x^p \left(1 - \left(1 - \frac{1}{x} \right)^{p-1} \right) = x^p \left(1 - \left(\sum_{k=0}^{\infty} \binom{p-1}{k} \left( -\frac{1}{x} \right)^{k} \right) \right) \\
        &= x^p \left(\sum_{k=1}^{\infty} (-1) \binom{p-1}{k} \left( -\frac{1}{x} \right)^{k} \right) = x^p \left(\sum_{k=1}^{\infty} (-1)^{k+1} \binom{p-1}{k} \left( \frac{1}{x} \right)^{k} \right)
    \end{align*}
    using the generalized binomial notation. The binomial series expansion in the second step is valid as $\left| \frac{1}{x} \right| < 1$. Furthermore, as $0 < p-1 < 1$, $\binom{p-1}{k}$ is negative for all even positive integers $k$ and positive for all odd positive integers $k$, so the expression $(-1)^{k+1} \binom{p-1}{k} \left( \frac{1}{x} \right)^{k}$ is positive for all positive integers $k$. As such, our expression is bounded below by \[x^p \left( \binom{p-1}{1} \cdot \frac{1}{x} \right) = x^{p-1} (p-1) \ge p-1\] when $x \ge 2$. 
\end{proof}

Now, we make the following definition.

\begin{defin}
    For a fixed $p \ge 1$ and a set $S = \{(u_i, v_i) : 1 \le i \le k\}$ of at least two points in $[0, 1] \times \mathbb{R}$ with $u_i < u_{i+1}$ for each $1 \le i \le k-1$, define \[H_{p, S} = \sum_{i=1}^{k-1} |v_{i+1}-v_{i}| \left(1 - |u_{i+1}-u_{i}|^{p-1} \right) .\]
\end{defin}

This seemingly contrived expression can be rewritten in various ways. For example, if we let $m_i = \frac{v_{i+1}-v_i}{u_{i+1}-u_i}$ be the slope of the line segment between $u_i$ and $u_{i+1}$ in $f_S$ for all $1 \le i \le k-1$, we have that \[H_{p, S} = \sum_{i=1}^{k-1} \left(|v_{i+1}-v_i| -|m_i| |u_{i+1}-u_i|^p \right) = J_1[f_S] - \left(\sum_{i=1}^{k-1} |m_i| |u_{i+1}-u_i|^p \right). \]

A key feature of $H_{p, S}$ which will be of later use is that whenever $f_S \in \mathcal{F}_1$, we have $0 \le H_{p, S} \le 1$. Indeed, when $J_1[f_S] \le 1$, we have \[H_{p, S} = J_1[f_S] - \left(\sum_{i=1}^{k-1} |m_i| |u_{i+1}-u_i|^p \right)  \le J_1[f_S] \le 1\] and \[H_{p, S} = J_1[f_S] -\left(\sum_{i=1}^{k-1} |m_i| |u_{i+1}-u_i|^p \right) \ge J_1[f_S] -\left(\sum_{i=1}^{k-1} |m_i| \right)  = 0.\] 

We prove that $H_{p, S}$ never decreases when we add new points into set $S$, and prove a lower bound on the amount that $H_{p, S}$ increases by upon the addition of a new point into $S$ in terms of $p$, the slope of $f_S$ at the $x$-coordinate of the new point, and the closest $x$-coordinate difference between the new point and a point in $S$. Note that these are precisely the quantities involved in the second inequality in Lemma \ref{partone}.

\begin{lem}\label{twopointfive}
    Consider a fixed real parameter $p \ge 1$, a set $S = \{(u_i, v_i) : 1 \le i \le k\}$ of at least two points in $[0, 1] \times \mathbb{R}$ with $u_i < u_{i+1}$ for each $1 \le i \le k-1$, and another $(x, y) \in [0, 1] \times \mathbb{R}$ with $x \neq u_i$ for any $1 \le i \le k$. If we let $d = \min_{1 \le i \le k}|x-u_i|$ and let $m = f'_S(x)$, which is the slope of the linear interpolation of $S$ at $x$, then we have \[H_{p, S \cup (x, y)} - H_{p, S} \ge (p-1) |m| d^p.\]
\end{lem}

\begin{proof}
    First, if $x < u_1$, it suffices to prove that $H_{p, S \cup (x, y)} - H_{p, S} \ge 0$ as $m=0$. Indeed, we have \[H_{p, S \cup (x, y)} - H_{p, S} = |y-v_1|(1 - |x - u_1|^{p-1}) \ge 0.\] 
    
    The case where $x > u_k$ resolves similarly. Now, let $u_i < x < u_{i+1}$ for some $1 \le i \le k-1$. Without loss of generality assume that $v_i \le v_{i+1}$. For convenience, define \[\Delta_1 = \sum_{j=1}^{i-1}|v_{j+1}-v_j|(1 - |u_{j+1} - u_j|^{p-1}), \hspace{2ex} \Delta_2 = \sum_{j=i+1}^{k-1}|v_{j+1}-v_j|(1 - |u_{j+1} - u_j|^{p-1}).\]
    
    First, we claim that we can reduce to only proving the case where $y \in [v_i, v_{i+1}]$. Indeed, note that if $y > v_{i+1}$, we can reduce it to the case where $y=v_{i+1}$, as 
    \begin{align*}
        H_{p, S \cup (x, y)}  &= \Delta_1 + \Delta_2 +\sum_{j=i}^{i+1} |y - v_j|(1 - |x - u_j|^{p-1}) \\
        &\ge  \Delta_1 + \Delta_2 +\sum_{j=i}^{i+1} |v_{i+1} - v_j|(1 - |x - u_j|^{p-1}) = H_{p, S \cup (x, v_{i+1})},
    \end{align*}
    where the second step follows from the fact that $|y-v_i| \ge |v_{i+1} - v_i|$ and $|y - v_{i+1}| \ge 0 = |v_{i+1}- v_{i+1}|$. In the same way, the case where $y < v_{i}$ can be reduced to $y=v_i$. Assume from now on that $y \in [v_i, v_{i+1}]$. Also, assume that $d = |x - u_i| \le |x-u_{i+1}|$. The case where $x$ is nearer to $u_{i+1}$ than $u_i$ can be handled similarly. As $|y- v_i| + |y-v_{i+1}| = |v_{i+1} - v_i|$ and $1 - |x - u_{i+1}|^{p-1} \le 1 - |x - u_i|^{p-1}$, 
    \begin{align*}
        H_{p, S \cup (x, y)} &= \Delta_1 + \Delta_2 +\sum_{j=i}^{i+1} |y - v_j|(1 - |x - u_j|^{p-1}) \\
        &\ge \Delta_1 + \Delta_2 + \left(|y - v_{i}| + |y - v_{i+1}| \right) (1 - |x - u_{i+1}|^{p-1}) \\
        &= \Delta_1 + \Delta_2 + |v_{i+1}-v_i| \left( 1 - |x - u_{i+1}|^{p-1} \right) \\
        &= \Delta_1 + \Delta_2 + \sum_{j=i}^{i+1} |v_i - v_j| (1 - |x - u_j|^{p-1}) = H_{p, S \cup (x, v_i)},
    \end{align*}
    so it suffices to check that $H_{p, S \cup (x, v_i)} - H_{p, S} \ge (p-1) |m| d^p$. Let $\frac{|u_i - u_{i+1}|}{|x-u_i|} = \frac{|u_i - u_{i+1}|}{d} = \lambda$. As such, we can substitute $|x-u_i|, |x-u_{i+1}|,$ and $|u_i - u_{i+1}|$ with $d$, $\lambda d - d$, and $\lambda d$, respectively. We have $\lambda \ge 2$, from $x$ being closer to $u_i$ than $u_{i+1}$. Note that as 
    \begin{align*}
    H_{p, S \cup (x, v_i)} &= \Delta_1 + \Delta_2 + |v_{i+1}-v_i|\left( 1 - |x - u_{i+1}|^{p-1} \right) \\
    &= \Delta_1 + \Delta_2 + \left| \frac{v_{i+1}-v_i}{u_{i+1}-u_i} \right| \left( |u_{i+1}-u_i| - |x - u_{i+1}|^{p-1} |u_{i+1}-u_i| \right)  \\
    &= \Delta_1 + \Delta_2 + |m| (\lambda d -(\lambda d - d)^{p-1}  (\lambda d)) 
    \end{align*}
    and 
    \begin{align*}
        H_{p, S} &= \Delta_1 + \Delta_2 + |v_{i+1}-v_i|\left( 1 - |u_i - u_{i+1}|^{p-1} \right) \\
        &= \Delta_1 + \Delta_2 + |m|\left( |u_i - u_{i+1}| - |u_i - u_{i+1}|^{p} \right) \\
        &= \Delta_1 + \Delta_2 + |m| \left( \lambda d - (\lambda d)^{p} \right),
    \end{align*}
    we have that indeed
    \begin{align*}
        H_{p, S \cup (x, v_i)} - H_{p, S} &= |m|(\lambda d)^{p} -|m| ( \lambda d - d)^{p-1} (\lambda d) \\
        &= |m| d^p \left(\lambda^{p} - (\lambda - 1)^{p-1} \cdot \lambda \right) \ge |m| d^p (p-1),
    \end{align*}
    where the last step follows from Lemma \ref{twovariable}.

\end{proof}

Now, we establish the following key inequality.

\begin{lem}\label{twopointsix}
    Choose a sequence $(u_1, v_1), (u_2, v_2), \ldots$ of points in $[0, 1] \times \mathbb{R}$ with $u_i \neq u_j$ for any $i \neq j$, and define the set $S_n = \{(u_i, v_i) : 1 \le i \le n\}$ for every positive integer $n$. For each $i >1$, let $d_i = \min_{j < i} |u_i - u_j|$ and let $m_i = f_{S_{i-1}}' (u_i)$, the slope of the linear interpolation of the first $i-1$ points of the sequence at $u_i$. If the inequality $J_1[f_{S_i}] \le 1$, or equivalently $f_{S_i} \in \mathcal{F}_1$, holds for every positive integer $i$, then for any $p>1$, we have \[\sum_{i=2}^{\infty} |m_i| d_i^p \le \frac{1}{p-1}.\]
\end{lem}

\begin{proof}
    Indeed, we have
    \begin{align*}
        \sum_{i=2}^{\infty} (p-1) |m_i| d_i^p  &= (p-1) |m_i| d_i^p + \sum_{i=3}^{\infty} (p-1) |m_i| d_i^p \\
        &= \sum_{i=3}^{\infty} (p-1) |m_i| d_i^p \le \sum_{i=3}^{\infty} \left(H_{p, S_{i}} - H_{p, S_{i-1}} \right) \le 1.
    \end{align*}
    
    Specifically, the second step follows from the fact $|m_2| = 0$, the third step follows from Lemma \ref{twopointfive}, and the last step follows from $0 \le H_{p, S_i} \le 1$ for any $i$, as $J_1[f_{S_i}] \le 1$ for any $i$.     
\end{proof}

We now combine Lemma \ref{twopointsix} with Lemma \ref{partone} to prove our desired upper bound on the error of the LININT learning algorithm when $(p, q) = (1+ \epsilon, 1+ \epsilon)$ for $\epsilon \in (0, 1)$.

\begin{thm}\label{twopointtwo}
    For $\epsilon \in (0, 1)$, we have $\mathscr L_{1+\epsilon}(\emph{\text{LININT}},\mathcal F_{1+\epsilon}) = O(\epsilon^{-1})$ for $\epsilon \in (0,1)$. 
\end{thm}

\begin{proof}
    Let $1 + \epsilon = q \in (1, 2)$ and suppose that $f \in \mathcal{F}_q$ is the function to be learned by LININT. Let $\sigma=(x_0, x_1, \ldots, x_n)$ be the sequence of inputs, all different from each other, with $n \ge 1$. Similar to \cite{geneson}, we define $S_i = \{(x_0, f(x_0)), \ldots, (x_n, f(x_n))\}$ for each $0 \le i \le n$, and let $\hat y_i \in \mathbb{R}$ for every $1 \le i \le n$ be the prediction of LININT corresponding to each $x_i$. Define the raw error of each round $e_i = |\hat y_i - f(x_i)|$, for each $1 \le i \le n$. Also, let $d_i = \min_{j<i} |x_i - x_j|$ and $m_i = f'_{S_{i-1}} (x_i)$ for each $i$. By Lemma \ref{partone}, for each round $1 \le i \le n$, we have either \[J_q[f_{S_i}] -  J_q[f_{S_{i-1}}] \ge \frac{q-1}{3} \cdot e_i^q\] or \[J_q[f_{S_i}] -  J_q[f_{S_{i-1}}] \ge \frac{q-1}{3 |m_i|^{2-q} \cdot d_i} \cdot e_i^2. \] Let set $I_1 \subseteq \{1, 2, \ldots, n\}$ consist of all positive integers $1 \le i \le n$ such that $e_i$ satisfy the first inequality. Analogously, let set $I_2 \subseteq \{1, 2, \ldots, n\}$ consist of all positive integers $1 \le i \le n$ such that $e_i, m_i, d_i$ satisfy the second inequality. Note that $I_1 \cup I_2 = \{1, 2, \ldots n\}$. 

    By Lemma \ref{linintmin} and \ref{partone}, we have \[1 \ge J_q[f] \ge J_q[f_{S_n}] = \sum_{i=1}^n (J_q[f_{S_i}] - J_q[f_{S_{i-1}}]) \ge \sum_{i \in I_1} (J_q[f_{S_i}] - J_q[f_{S_{i-1}}]) \ge \frac{q-1}{3} \cdot \sum_{i \in I_1} e_i^q.\] 

    Dividing both sides, we get $\sum_{i \in I_1} e_i^q \le \frac{3}{q-1}.$ Similarly, we also have \[1 \ge \sum_{i \in I_2}(J_q[f_{S_i}] - J_q[f_{S_{i-1}}]) \ge \frac{q-1}{3} \sum_{i \in I_2} \frac{e_i^2}{|m_i|^{2-q} \cdot d_i}.\] 
    
    Thus, $\sum_{i \in I_2} \frac{e_i^2}{|m_i|^{2-q} \cdot d_i} \le \frac{3}{q-1}$. Now, note that by Hölder's Inequality, 
    \begin{align*}
        \sum_{i \in I_2} e_i^q &= \sum_{i \in I_2} \left(\frac{e_i^q}{|m_i|^{\frac{q(2-q)}{2}} \cdot d_i^{\frac{q}{2}}} \right) \cdot \left( |m_i|^{\frac{q(2-q)}{2}} \cdot d_i^{\frac{q}{2}} \right) \\
        &\le \left(\sum_{i \in I_2} \frac{e_i^2}{|m_i|^{2-q} \cdot d_i} \right)^{\frac{q}{2}} \left(\sum_{i \in I_2} |m_i|^q d_i^{\frac{q}{2-q}} \right)^{\frac{2-q}{2}} \le \left(\frac{3}{q-1} \right)^{\frac{q}{2}} \left(\sum_{i \in I_2} |m_i|^q d_i^{\frac{q}{2-q}} \right)^{\frac{2-q}{2}}.
    \end{align*}

    Note that $|m_i| \cdot d_i^{\frac{1}{2-q}} \le |m_i| \cdot  d_i$, as $\frac{1}{2-q} > 1$. Furthermore, if $1 \le j \le i-1$ satisfies that $|x_i - x_j| = \min_{1 \le k \le i-1} |x_i - x_k|$, then \[|m_i| \cdot d_i = \left|\frac{f_{S_{i-1}}(x_i) - f_{S_{i-1}}(x_j)}{x_i - x_j } \right| \cdot |x_i - x_j| = |f_{S_{i-1}}(x_i) - f_{S_{i-1}}(x_j)| \le 1,\] as the range of $f_{S_{i-1}}$ is at most $1$ as $f_{S_{i-1}} \in \mathcal{F}_q$. As such, $|m_i| \cdot d_i^{\frac{1}{2-q}} \le 1$ for each $i$, so 
    
    \begin{align*}
        \sum_{i \in I_2} |m_i|^q d_i^{\frac{q}{2-q}} &= \sum_{i \in I_2} \left(|m_i| d_i^{\frac{1}{2-q}} \right)^q \le \sum_{i \in I_2} |m_i| d_i^{\frac{1}{2-q}} \le |m_1| d_1^{\frac{1}{2-q}} + \sum_{i=2}^n |m_i| d_i^{\frac{1}{2-q}} \\
        &\le 1 + \frac{1}{\frac{1}{2-q}-1} = 1 + \frac{2-q}{q-1} = \frac{1}{q-1},
    \end{align*}
    where the fourth step follows from Lemma \ref{twopointsix}, which is valid because $f_{S_i} \in \mathcal{F}_q$ which is in $\mathcal{F}_1$ for all $i$. Therefore, we have \[\sum_{i \in I_2} e_i^q \le \left(\frac{3}{q-1} \right)^{\frac{q}{2}} \left(\frac{1}{q-1} \right)^{\frac{2-q}{2}} \le \frac{3}{q-1} .\] Putting everything together, we have \[\mathscr L_{q}(\text{LININT},f, \sigma) =\sum_{i=1}^n e_i^q \le \sum_{i \in I_1} e_i^q + \sum_{i \in I_2} e_i^q \le \frac{6}{q-1},\] for any choice of $f \in \mathcal{F}_q$, any positive integer $n$, and any sequence of inputs $\sigma \in [0, 1]^{n+1}$. As such, we have $\mathscr L_{1+\epsilon}(\text{LININT}, \mathcal{F}_{1+\epsilon}) \le \frac{6}{\epsilon}$, for any $\epsilon \in (0, 1)$. 
\end{proof}

Note that from Geneson and Zhou \cite{geneson}, we have the following inequality. 

\begin{lem}[\cite{geneson}]
    For any $q>1$ and $p' > p > 1$, we have $\mathscr L_{p'}(\emph{\text{LININT}},\mathcal F_{q}) \le \mathscr L_{p}(\emph{\text{LININT}},\mathcal F_{q})$.
\end{lem}

Combining this with Theorem \ref{twopointtwo}, we can get the following. 

\begin{thm}
    For $0 < \epsilon \le \delta < 1$, we have $\mathscr L_{1+\delta}(\emph{\text{LININT}},\mathcal F_{1+\epsilon}) = O(\epsilon^{-1})$ for $\epsilon \in (0,1)$. 
\end{thm}

As such, we have the following upper bound.

\begin{thm}\label{epsilonless}
    For $0 < \epsilon \le \delta < 1$, we have $\opt_{1+\delta}(\mathcal{F}_{1+\epsilon}) = O(\epsilon^{-1})$.
\end{thm}

On the other hand, as $\mathcal{F}_q' \subseteq \mathcal{F}_q$ for all $q' > q \ge 1$ by Jensen's inequality, we have $\opt_p(\mathcal{F}_{q'}) \le \opt_p(\mathcal{F}_q)$ for all $q' > q \ge 1$. Thus, for all $0 < \delta \le \epsilon <1$, we have $\opt_{1+\delta}(\mathcal{F}_{1+\epsilon}) \le \opt_{1+\delta}(\mathcal{F}_{1+\delta}) = O(\delta^{-1})$. Combining this with Theorem \ref{epsilonless}, we have the following upper bound. 

\begin{thm}
    For $\delta, \epsilon \in (0, 1)$, we have $\opt_{1+\delta}(\mathcal{F}_{1+\epsilon}) = O(\min(\delta, \epsilon)^{-1})$.
\end{thm}

As a corollary, we have now completely classified the values of $p, q \ge 1$ such that $\opt_p(\mathcal{F}_q)$ is finite, confirming a conjecture by \cite{geneson}.

\begin{cor}\label{twopointtwelve}
    The worst-case learning error $\opt_p(\mathcal{F}_q)$ is finite if and only if $p, q > 1$. 
\end{cor}

\section{Online Learning of Polynomials}\label{2}

We prove the conjecture, raised by \cite{geneson}, that for all $p > 0$ and $q \ge 1$, $\opt_p(\mathcal{P}_q) = \opt_p(\mathcal{F}_q)$. Intuitively, this means that given a fixed $q$-action restriction on a smooth function, having the extra restriction of the function being a polynomial does not decrease the learner's error. The main idea of our proof is that the adversary can replicate any strategy it uses against the learner in the scenario of all smooth functions to the scenario of all smooth polynomials. As such, it can guarantee the same worst-case error in the latter scenario as in the former. Note that in this section, we are not concerned with the learner's strategy, whether it be LININT or another algorithm; we focus only on strategies used by the adversary against the learner, adapting based on whatever predictions the learner makes to maximize error for the learner. 

Our proof invokes the following well-known result on the polynomial approximation of continuous real-valued functions. 

\begin{thm}[Weierstrass Approximation Theorem]\label{weierstrass}
    For any continuous real-valued function $f$ defined on a real interval $[a, b]$ and any $\epsilon > 0$, there exists a polynomial $P$ such that for all $x \in [a, b]$, $|f(x)-P(x)| < \epsilon$. 
\end{thm}

Using the Weierstrass Approximation Theorem, we prove that given any set of known points $S$, the adversary can find a polynomial $P$ that is arbitrarily close to the points in $S$ and has $q$-action bounded above by $\epsilon$ more than the $q$-action of $f_S$ for any $\epsilon >0$. 

\begin{lem}\label{2.1}
    Given any $q \ge 1$, $\epsilon > 0$, and a set $S = \{(u_i, v_i): 1 \le i \le m\}$ of points in $[0, 1] \times \mathbb{R}$ such that $u_1 < \cdots < u_m$ and its corresponding linear interpolation function $f_S$, there exists a polynomial $P: [0, 1] \rightarrow \mathbb{R}$ such that $|P(u_i) - v_i| = |P(u_i) - f_S(u_i)| < \epsilon$ for all $1 \le i \le m$, and \[J_q[P] < J_q[f_S] + \epsilon.\]
\end{lem}

\begin{proof}
    Consider the derivative of $f_S$, which when considered as a graph, consists of disjoint horizontal line segments, broken off at each $u_i$, where it is not defined. We modify $f_S'$ to make it continuous, so Theorem \ref{weierstrass} can be used. Define $d_0, d_1, \ldots, d_m$ to be the slopes of the segments of $f_S$ in order:
    \[d_i = \begin{cases}
        0 & i = 0 \\
        \frac{v_{i+1}-v_i}{u_{i+1}-u_i} & 1 \le i \le m-1 \\
        0 & i = m.
    \end{cases}\]
    
    Now, for a sufficiently small positive $\epsilon_2 < \min_{1 \le i <  j \le m} |u_i - u_j|$, define the function $F'_S : [0, 1] \rightarrow \mathbb{R}$ such that  \[F_S'(x) = \begin{cases}
        d_0 & x \le u_1 - \epsilon_2 \\ d_{i-1} + \frac{(x - (u_i - \epsilon_2))(d_{i}-d_{i-1})}{\epsilon_2} & x \in (u_i-\epsilon_2,u_{i}] \\ d_i & x \in (u_{i}, u_{i+1}-\epsilon_2] \\
        d_m & x > u_m.
    \end{cases}\]

    It is easy to see that $F'_S$ is continuous for all $x \in (0, 1)$. Indeed, intuitively, the graph of $F'_S$ is essentially the graph of $f'_S$ but with the disjoint horizontal segments ``connected".
    
    We claim that for sufficiently small $\epsilon_2$, $\int_{0}^x |F'_S(t)|^q \text{d}t$ gets arbitrarily close to $ \int_{0}^x |f'_S(t)|^q \text{d}t$ for any $x \in [0, 1]$ and any $q \ge 1$. Indeed, $F'_S(t)$ only differs from $f'_S(t)$ when $t \in (u_i - \epsilon_2, u_i]$ for some $1 \le i \le m$. For a fixed $i$, when $t$ is in this range, $f'_S(t)$ is precisely $ d_{i-1}$, whereas $F'_S(t)$ is between $d_{i-1}$ and $d_i$. As such, $||f'_S(t)|^q - |F'_S(t)|^q|$ is bounded above by $c = \max_{1 \le i < j \le m} \left( |d_i|^q - |d_j|^q \right)$, which is finite, for any $q$. As such, for any $x \in [0, 1]$,

    \begin{align*}
        \left| \int_{0}^x |F'_S(t)|^q \text{d}t - \int_{0}^x |f'_S(t)|^q \text{d}t \right|  &\le   \int_{0}^x \left ||f'_S(t)|^q - |F'_S(t)|^q\right| \text{d}t \\
        &\le \sum_{i=1}^m  \int_{u_i - \epsilon_2}^{u_i} \left ||f'_S(t)|^q - |F'_S(t)|^q \right|\text{d}t  \\
        &\le \sum_{i=1}^m \int_{u_i -\epsilon_2}^{u_i} c \text{d}t = mc\epsilon_2,
    \end{align*}

    \noindent which gets arbitrarily close to $0$ if we set $\epsilon_2$ sufficiently small. We will use this inequality later on. Note that as a corollary, setting $x=1$ above yields that $\int_{0}^1 |F'_S(t)|^q \text{d}t$ gets arbitrarily close to $J_q[f_S]$ when we set $\epsilon_2$ sufficiently small—we will also use this fact later on.
    
    Now, for sufficiently small $\epsilon_3 > 0$ (which we will specify later), consider a polynomial $Q$ such that for all $x \in [0, 1]$, $|F'_S(x) - Q(x)| < \epsilon_3$, which is possible by Theorem \ref{weierstrass} as $F'_S$ is continuous for all $x \in [0, 1]$. Subsequently, let $P(x) \equiv \int_0^{x} Q(t) \text{d}t + v_1$, which is clearly a polynomial. 
    
    We claim that $J_q[P] < J_q[f_S] + \epsilon$ for any $\epsilon > 0$, given that our choices of $\epsilon_2$ and $\epsilon_3$ are sufficiently small. In particular, define $c_1 = \max_{x \in [0, 1]}|F'_S(x)| = \max_{0 \le i \le m} |d_i|$. Now, pick $\epsilon_2 < \frac{\epsilon}{2mc}$ and let $\epsilon_3$ be sufficiently small such that $(c_1 + \epsilon_3)^q - c_1^q < \frac{\epsilon}{2}$. Note that this implies that $(|F'_S(x)| + \epsilon_3)^q - (|F'_S(x)|)^q < \frac{\epsilon}{2}$ for all $x \in [0, 1]$, as $0 \le |F'_S(x)| \le c_1$ and the function $(y+\epsilon_3)^q - y^q$ is increasing for all $y \ge 0$ and any positive $\epsilon$.

    As such, we have the upper bound
    \begin{align*}
        J_q[P] &= \int_0^1 |Q(x)|^q \text{d}x \le \int_0^1 (|F'_S(x)| + \epsilon_3)^q \text{d}x \le \int_0^1 \left(|F'_S(x)|^q + \frac{\epsilon}{2} \right) \text{d}x \\
        &= \int_0^1 |F'_S(x)|^q \text{d}x + \frac{\epsilon}{2} \le J_q[f_S] + \left(\int_0^1 |F'_S(x)|^q \text{d}x - J_q[f_S] \right) + \frac{\epsilon}{2} \\
        &= J_q[f_S] + \left(\int_0^1 |F'_S(x)|^q \text{d}x - \int_{0}^1 |f'_S(x)|^q \text{d}x \right) + \frac{\epsilon}{2} \\
        &\le J_q[f_S] + mc\epsilon_2 +  \frac{\epsilon}{2} < J_q[f_S] + \frac{\epsilon}{2} + \frac{\epsilon}{2} = J_q[f_S] + \epsilon.
    \end{align*}

    Furthermore, we also claim that $|P(u_i) - v_i| < \epsilon$ for any $\epsilon >0$, given that our choices of $\epsilon_2$ and $\epsilon_3$ are sufficiently small. Indeed, as long as $\epsilon_3 < \frac{\epsilon}{2}$ and $\epsilon_2 < \frac{\epsilon}{2mc}$, we have for each $1 \le i \le m$ that
    \begin{align*}
       P(u_i) &= \int_0^{u_i} Q(t) \text{d}t + v_1 \le \int_0^{u_i} \left(F'_S(t) + \epsilon_3 \right) \text{d}t + v_1 \le \int_0^{u_i} \left(F'_S(t)\right) \text{d}t + \epsilon_3 + v_1 \\
       &\le \int_0^{u_i} \left(f'_S(t)\right) \text{d}t + mc\epsilon_2 + \epsilon_3 + v_1 = (v_i - v_1) + v_1 + mc\epsilon_2 + \epsilon_3  \\
       &= v_i + mc\epsilon_2 + \epsilon_3 < v_i + \frac{\epsilon}{2} + \frac{\epsilon}{2} = v_i + \epsilon
    \end{align*}
    and 
    \begin{align*}
       P(u_i) &= \int_0^{u_i} Q(t) \text{d}t + v_1 \ge \int_0^{u_i} \left(F'_S(t) - \epsilon_3 \right) \text{d}t + v_1 \ge \int_0^{u_i} \left(F'_S(t) \right) \text{d}t - \epsilon_3 + v_1 \\
       &\ge \int_0^{u_i} \left(f'_S(t)\right) \text{d}t - mc\epsilon_2 - \epsilon_3 + v_1 = (v_i - v_1) + v_1 - mc\epsilon_2 - \epsilon_3  \\
       &= v_i - mc\epsilon_2 - \epsilon_3 < v_i - \frac{\epsilon}{2} - \frac{\epsilon}{2} = v_i - \epsilon.
    \end{align*}
    Therefore, for any $\epsilon > 0$, we can find a polynomial $P$ such that $|P(u_i)-v_i|< \epsilon$ for all $1 \le i \le m$ and $J_q[P] < J_q[f_S] + \epsilon$. \\
\end{proof}

For a finite set of points $S \subseteq [0, 1] \times \mathbb{R}$, we now prove a fact about constructing a function that passes through all points in $S$ by taking a weighted average of a special set of $2^{|S|}$ functions that each do not necessarily pass through the points in $S$. We will later combine this with Lemma \ref{2.1} to implicitly construct a polynomial passing through all points in $S$ by taking a weighted average of $2^{|S|}$ polynomials that each approximate but do not necessarily pass through the points in $S$. We use the convention that $[m] := \{1, 2, \ldots, m\}$.

\begin{lem}\label{weightedaverage}
    Suppose there is a set $S = \{(u_i, v_i): 1 \le i \le m\}$ of $m$ points in $[0, 1] \times \mathbb{R}$ such that $u_1 < \ldots < u_m$. If we have $2^m$ functions $f_X : [0, 1] \times \mathbb{R}$, corresponding to each subset $X \subseteq [m]$, such that for any $X \subseteq [m]$, it holds for any integer $1 \le i \le m$ with $i \in X$ that $f_X(u_i) > v_i$, and it holds for any integer $1 \le i \le m$ with $i \notin X$ that $f_X(u_i) < v_i$, then there exists a weighted average of the functions $f \equiv \sum_{X \subseteq [m]} w_X f_X$, with $0 \le w_X \le 1$ for all $X \subseteq [m]$ and $\sum_{X \subseteq [m]} w_X = 1$, such that $f(u_i) = v_i$ for all $1 \le i \le m$. 
\end{lem}

\begin{proof}
    We proceed by induction. For the base case of $m=1$, we are given two functions $f_{\emptyset}, f_{\{1\}} : [0, 1] \rightarrow \mathbb{R}$, and one point $(u_1, v_1)$ that the desired weighted average should pass through. As $f_{\{1\}}(u_1) > v_1$ and $f_{\emptyset}(u_1) < v_1$, we can clearly assign two weights $0 \le w_{\{1\}}, w_{\emptyset} \le 1$ with sum $1$ such that $w_{\{1\}}f_{\{1\}}(u_1) + w_{\emptyset}f_{\emptyset} = v_1$. \\

    Proceeding with the inductive step, we assume that the result is true for $m-1$, and prove that it is true for $m$ as well. Let the set of points be $S = \{(u_i, v_i): 1 \le i \le m\}$, and functions be $f_X: [0, 1] \times \mathbb{R}$ where $X \subseteq [m]$. Partition the $2^m$ functions into two groups, $A$ and $B$, such that $f_X$ goes into $A$ if $m \in X$, and goes into $B$ if $m \notin X$. Clearly, each of the two groups have $2^{m-1}$ functions. Note that for every subset $Y \subseteq [m-1]$, there exists a function in each of $A$ and $B$ such that the output at $u_i$ is greater than $v_i$ precisely when $i \in Y$, and the output at $u_i$ is less than $v_i$ precisely when $i \notin Y$ (specifically, the function $f_{Y \cup \{m\}}$ in $A$, and the function $f_Y$ in $B$). By the inductive hypothesis, there exists a weighted average $F_1 \equiv \sum_{X \subseteq [m], m \in X} w_X f_X$ of the functions in $A$ such that $F_1 (u_i) = v_i$ for all $1 \le i \le m-1$. Similarly, there exists a weighted average $F_2 \equiv \sum_{X \subseteq [m], m \notin X} w_X f_X$ of the functions in $B$ such that $F_2 (u_i) = v_i$ for all $1 \le i \le m-1$. \\

    Note that as all of the functions in group $A$ satisfy $f(u_m) > v_m$, the weighted average of these functions, $F_1$, must satisfy $F_1(u_m) > v_m$ as well. Similarly, as all of the functions in the functions in group $B$ satisfy $f(u_m) < v_m$, the weighted average, $F_2$, must satisfy $F_2(u_m) < v_m$. As such, there exists a weighted average of the functions $F_1$ and $F_2$, $f \equiv W_1 F_1 + W_2 F_2$, with $0 \le W_1, W_2 \le 1$ and $W_1 + W_2 = 1$, such that $f(u_m) = v_m.$ Furthermore, note that as $F_1(u_i) = F_2(u_i) = v_i$ for all $1 \le i \le m-1$, $f(u_i) = v_i$ for all $1 \le i \le m-1$ as well. Thus, $f(u_i) = v_i$ for all $ 1 \le i \le m$. Furthermore, as $f$ is the weighted average of two weighted averages of functions $f_X$ where $X \subseteq [m]$, $f$ is a weighted average of functions $f_X$ where $X \subseteq [m]$, and is hence our desired function.
\end{proof}

We now establish that any function that is a weighted average of a set of functions that each have $q$-action less than $1$ also has $q$-action less than $1$. 

\begin{lem}\label{powermean}
    For any $q \ge 1$, if we have $n \ge 1$ continuous functions $f_1, f_2, \ldots, f_n : [0, 1] \rightarrow \mathbb{R}$ such that $J_q[f_i] < 1$ for all $1 \le i \le n$, then for any weighted average of the functions $f \equiv \sum_{i=1}^n w_i f_i$, with $0 \le w_i \le 1$ for all $1 \le i \le n$ and $\sum_{i=1}^n w_i = 1$, we have $J_q[f] < 1.$
\end{lem}

\begin{proof}
    We can see that
    \begin{align*}
        J_q[f] &= \int_0^1 |f'(x)|^q \text{d}x = \int_0^1 \left| \sum_{i=1}^n w_i f_i'(x) \right|^q \text{d}x \le \int_0^1 \left( \sum_{i=1}^n \left| w_i \right| \left|f_i'(x) \right| \right)^q \text{d}x\\
        &\le \int_0^1 \left(\left( \sum_{i=1}^n w_i \left|f_i'(x) \right|^q \right)^\frac{1}{q} \right)^q \text{d}x = \int_0^1 \left( \sum_{i=1}^n w_i \left|f_i'(x) \right|^q \right) \text{d}x \\
        &= \sum_{i=1}^n \left(\int_0^1 w_i \left|f_i'(x) \right|^q  \text{d}x \right) = \sum_{i=1}^n w_i J_q[f_i] < \sum_{i=1}^n w_i = 1,
    \end{align*}
where the fourth step follows by the Weighted Power Mean Inequality, as $q \ge 1$. 
\end{proof}

Putting everything together, we have the following key result.

\begin{lem}\label{epsilonpolynomial}
    Given a set $S = \{(u_i, v_i) : 1 \le i \le m\}$ of $m$ points in $[0, 1] \times \mathbb{R}$ with $u_1 < \cdots < u_m$ and any $q \ge 1$, if there exists an absolutely continuous function $f : [0, 1] \rightarrow \mathbb{R}$ with $J_q[f] < 1$ such that $f(u_i) = v_i$ for all $1 \le i \le m$, then there exists a polynomial $P : [0, 1] \rightarrow \mathbb{R}$ with $J_q[P] < 1$ such that $P(u_i) = v_i$ for all $1 \le i \le m$. 
\end{lem}

\begin{proof}
    Assume that $m \ge 2$, as the case where there is only one point in $S$ is trivial. Note that by Lemma \ref{linintmin}, the condition that $J_q[f] < 1$ implies that $J_q[f_S] < 1$. Now, pick a  small $\epsilon_1 > 0$ (which we will specify later), and define the set $S_X$, for each subset $X \subseteq [m]$, as follows: \[S_X = \{(u_i, v_i + \epsilon_1): 1 \le i \le m, i \in X\} \cup \{(u_i, v_i - \epsilon_1): 1 \le i \le m, i \notin X\}.\] As such, each $S_X$ contains exactly $m$ points, each of whose $y$-values are within $\epsilon_1$ from the $y$-values of the corresponding points in $S$. For each $X$, we claim that setting $\epsilon_1>0$ sufficiently small, we can guarantee $J_q[f_{S_X}] < 1$. Indeed, if we let $d_0, d_1, \ldots, d_m$ denote the slopes of the segments of $f_S$ in order, and define $d_0', d_1', \ldots, d_m'$ to be the slopes of the segments of $f_{S_X}$ in order, notice that $d_0' = d_m' = 0$, and for each $1 \le i \le m-1$, we have either $d_i' = d_i$, $d_i' = d_i + \frac{2\epsilon_1}{u_{i+1}-u_i}$, or $d_i' = d_i - \frac{2\epsilon_1}{u_{i+1}-u_i}$. As such, in any case, we have \[|d_i'| \le |d_i| + \frac{2\epsilon_1}{u_{i+1} - u_i} \le |d_i| + \frac{2\epsilon_1}{\min_{1 \le i \le m-1} |u_{i+1} - u_i| } = |d_i| + C \epsilon_1,\] if we substitute $C = \frac{2}{\min_{1 \le i \le m-1} |u_{i+1} - u_i|}$, which is fixed for a set $S$. Notice that for each $i$, we can set $\epsilon_1$ sufficiently small such that the inequality $\left(|d_i| + C\epsilon_1 \right)^q - |d_i|^q < \frac{1 - J_q[f_S]}{m}$ holds, as $\frac{1 - J_q[f_S]}{m}$ is positive. Hence, let $\epsilon_1$ be sufficiently small such that the inequality $\left(|d_i| + C\epsilon_1 \right)^q - |d_i|^q < \frac{1 - J_q[f_S]}{m}$ holds for each $1 \le i \le m-1$. 
    
    Now, notice that
    \begin{align*}
    J_q[f_{S_X}] - J_q[f_S] &= \sum_{i=1}^{m-1} |u_{i+1} - u_i| |d_i'|^q - \sum_{i=1}^{m-1} |u_{i+1} - u_i| |d_i|^q = \sum_{i=1}^{m-1} |u_{i+1} - u_i| (|d_i'|^q - |d_i|^q) \\
    &\le \sum_{1 \le i \le m-1, |d_i'| > |d_i|} |u_{i+1} - u_i| \left(|d_i'|^q - |d_i|^q \right)  \\ &\le \sum_{1 \le i \le m-1, |d_i'| > |d_i|} \left(|d_i'|^q - |d_i|^q \right) \\
    &\le \sum_{1 \le i \le m-1, |d_i'| > |d_i|} \left(\left(|d_i| + C\epsilon_1 \right)^q - |d_i|^q \right) < m \cdot \frac{1 - J_q[f_S]}{m} = 1 - J_q[f_{S}],
    \end{align*}
    implying that $J_q[f_{S_X}] < 1$. Therefore, we can pick sufficiently small $\epsilon_1$ such that for each $X \subseteq [m]$, we have $J_q[f_{S_X}] < 1$. 
    
    Now, for each $X \subseteq [m]$, pick a sufficiently small $\epsilon_2 > 0$ such that $\epsilon_2 < \epsilon_1$ and $\epsilon_2 < 1 - J_q[f_{S_X}]$, and let $P_X: [0, 1] \rightarrow \mathbb{R}$ be a polynomial such that $$|P_X(u_i) - f_{S_X}(u_i)| < \epsilon_2 \text{ for all } 1 \le i \le m \text{, and } J_q[P_X] < J_q[f_{S_X}] + \epsilon_2,$$ which must be possible by Lemma \ref{2.1}. By the first condition, we can see that whenever $1 \le i \le m$ satisfies $i \in X$, we have \[P_X(u_i) > f_{S_X}(u_i) - \epsilon_2 > f_{S_X}(u_i) - \epsilon_1 = v_i,\] and whenever $i \notin m$, we have \[P_X(u_i) < f_{S_X}(u_i) + \epsilon_2 < f_{S_X}(u_i) + \epsilon_1 = v_i. \] Furthermore, the second condition yields that $J_q[P_X] < 1$ as well. 
    
    As such, we have $2^m$ polynomials $P_X$, corresponding to each subset $X \subseteq [m]$, such that for each $X$, it holds for any integer $1 \le i \le m$ with $i \in X$ that $P_X(u_i) > v_i$, and it holds for any integer $1 \le i \le m$ with $i \notin X$ that $P_X(u_i) < v_i$. By Lemma \ref{weightedaverage}, there exists a weighted average of the polynomials, another polynomial, $P \equiv \sum_{X \subseteq [m]} w_X P_X$, with weights $0 \le w_X \le 1$ for each $X$ and $\sum_{X \subseteq [m]} w_X = 1$, such that $P(u_i) = v_i$ for all $1 \le i \le m$. Furthermore, by Lemma \ref{powermean}, as $J_q[P_X] < 1$ for all $X$, we know that $J_q[P] < 1$ as well. 
\end{proof}

Now, for any $q \ge 1$, let $\mathcal{F}'_q$ denote the class of functions $f: [0, 1] \rightarrow \mathbb{R}$ such that $J_q[f] < 1$. As such, $\mathcal{F}'_q \subseteq \mathcal{F}_q$. Similarly, let $\mathcal{P}'_q$ denote the class of polynomials $P: [0, 1] \rightarrow \mathbb{R}$ such that $J_q[P] < 1$. As such, $\mathcal{P}'_q \subseteq \mathcal{P}_q$ and $\mathcal{P}'_q \subseteq \mathcal{F}'_q$. We now establish the following key result that looks very close to our desired final result. 

\begin{lem}\label{actionlessthanone}
    For all $p>0$ and $q \ge 1$, we have $\opt_p(\mathcal{P}'_q) = \opt_p(\mathcal{F}'_q)$.
\end{lem}

\begin{proof}

Note that by Lemma \ref{epsilonpolynomial}, given any finite set $S$ of points in $[0, 1] \times \mathbb{R}$, there exists a polynomial $P \in \mathcal{P}'_q$ passing through all the points in $S$ if and only if there exists a general smooth function $f \in \mathcal{F}'_q$ passing through all the points in $S$. 

We claim that this means that the adversary can replicate any strategy that it uses in the learning scenario of $\mathcal{F}'_q$ to maximize the learner's error to the learning scenario of $\mathcal{P}'_q$, and vice versa. This would imply that the worst-case learning error is equal in the two scenarios.

To prove this, note that for any sequence of points $(x_0, y_0), (x_1, y_1), \ldots, (x_m, y_m)$ that the adversary reveals in the learning scenario of $\mathcal{F}'_q$, the same sequence of points can be revealed in the learning scenario of $\mathcal{P}'_q$ (and vice versa as well), as the existence of a general function $f \in \mathcal{F}'_q$ passing through the sequence of points implies the existence of a polynomial $P \in \mathcal{P}'_q$ passing through the same sequence of points (and vice versa as well, trivially). From here, it follows that $\opt_p(\mathcal{P}'_q) = \opt_p(\mathcal{F}'_q)$.
\end{proof}

Now, we prove that the following equality holds.
\begin{lem}
    For all $p>0$ and $q \ge 1$, we have $\opt_p(\mathcal{F}'_q) = \opt_p(\mathcal{F}_q)$.
\end{lem}

\begin{proof}
    For a real number $c>0$, define the family of functions $c\mathcal{F}_q$ to contain precisely the functions $g:[0, 1] \rightarrow \mathbb{R}$ such that $g \equiv cf$, where $f \in \mathcal{F}_q$. Note that $c\mathcal{F}_q$ contains exactly the continuous functions $g: [0, 1] \rightarrow \mathbb{R}$ with $J_q[g] \le c$. As such, for every $0 < c < 1$, we have that $c\mathcal{F}_q \in \mathcal{F}'_q$. Thus, $\opt_p(c\mathcal{F}_q) \le \opt_p(\mathcal{F}'_q)$ for any $0<c<1$. 

    Furthermore, we claim that \[\opt_p(c\mathcal{F}_q) = c^p \opt_p(\mathcal{F}_q)\] for any $c > 0$. Indeed, as the family $c\mathcal{F}_q$ is precisely the family $\mathcal{F}_q$ scaled by a factor of $c$, the raw prediction errors of each round in the learning scenario of $c\mathcal{F}_q$ are effectively scaled by a factor of $c$ from the raw prediction errors of the learning scenario of $\mathcal{F}_q$. Yet, as the raw prediction errors are raised to the $p^{\text{th}}$ power, the total error function gets scaled by a factor of $c^p$, yielding the equation. Thus, we have \[c^p \opt_p(\mathcal{F}_q) \le \opt_p(\mathcal{F}'_q)\] for all $0 < c < 1$. Taking $c \rightarrow 1$, we can see that  $\opt_p(\mathcal{F}'_q)$ cannot be any less than $\opt_p(\mathcal{F}_q)$. Yet, as $\mathcal{F}'_q \in \mathcal{F}_q$, we also have $\opt_p(\mathcal{F}'_q) \le \opt_p(\mathcal{F}_q)$. We thus have $\opt_p(\mathcal{F}'_q) = \opt_p(\mathcal{F}_q)$.  
\end{proof}

Using the same idea, we can get a similar result concerning $\mathcal{P}'_q$ and $\mathcal{P}_q$.

\begin{lem}
    For all $p > 0$ and $q \ge 1$, we have $\opt_p(\mathcal{P}'_q) = \opt_p(\mathcal{P}_q)$.
\end{lem}

Combining all of the above, we get our result.

\begin{thm}
    For any $p>0$ and $q \ge 1$, we have $\opt_p(\mathcal{P}_q) = \opt_p(\mathcal{F}_q)$.
\end{thm}

\section{Online Learning of Smooth Functions with Noisy Feedback}\label{3}

We start with the proof of Theorem \ref{1.5}, a fundamental result on the number of initial feedback rounds necessary in the noisy model. 

\begin{thm}
    For any integer $\eta \ge 1$, if incorrect feedback can be given up to $\eta$ times, then at least $2\eta+1$ initial rounds must be thrown out for the learner to guarantee finite error in its first prediction that counts towards error evaluation.
\end{thm}

\begin{proof}
    We first show that after receiving $2\eta$ initial rounds of feedback, the learner cannot guarantee finite error on its next trial. Consider any sequence of inputs $\sigma=(x_0, x_1, \ldots, x_{2\eta-1}) \in [0, 1]^{2\eta}$. Let the adversary claim that $f(x_i) = 0$ for all $0 \le i \le \eta-1$, and that $f(x_i) = C$ for all $\eta \le i \le 2\eta-1$ for some arbitrarily large $C$. Now, suppose the learner is queried on the value of $f(x_{2\eta})$. If $\hat y_{x_{2\eta}} \ge \frac{C}{2}$, then let the function be $f(x) \equiv 0$, which is valid because the adversary would have lied exactly $\eta$ times. Similarly, if $\hat y_{x_{2\eta}} \le \frac{C}{2}$, let the function be $f(x) \equiv C$, which would also be valid as it induces $\eta$ lies. In any case, the raw error $|\hat y_{x_{2\eta}} - f(x_{2\eta})| \ge \frac{C}{2}$, which can be made arbitrarily large. 

    Now, we establish that throwing out $2\eta+1$ initial rounds suffices. We prove a stronger statement, that after $2\eta+1$ initial rounds, the learner is able to bound the value of the function $f$ at any input within a closed interval of length $2$. Suppose that in the first $2\eta+1$ rounds, the learner receives the set of points $S=\{(x_i, y_i): 1 \le i \le 2\eta+1\}$, where $x_i$ and $y_i$ denote the queried input and adversary feedback, respectively, in the $i^{\text{th}}$ round. 

    Let $(v_1, v_2, \ldots, v_{2\eta+1})$ be a permutation of $(y_1, y_2, \ldots, y_{2\eta+1})$, such that $v_1 \le \cdots \le v_{2\eta+1}$. Now, note that as the adversary may lie at most $\eta$ times, at least $\eta+1$ points among the $2\eta+1$ points in $S$ must align with the actual function. As such, at least $\eta+1$ elements of $\{v_1, \ldots, v_{2\eta+1}\}$ must be in the range of $f$. Note that any $\eta+1$-element subset of $\{v_1, \ldots, v_{2\eta+1}\}$ must contain a member $v_j$ with $j \ge \eta+1$, so the range of $f$ must contain some $v_j \ge v_{\eta+1}$. As the difference between the greatest and least values of $f$ is at most $1$, from the fact that $f \in \mathcal{F}_q$, we must have $f(x) \ge v_j-1 \ge v_{\eta+1} - 1$ for all $x \in [0, 1]$. On the other hand, as any $\eta+1$-element subset of $\{v_1, \ldots, v_{2\eta+1}\}$ must also contain a member $v_k$ with $k \le \eta+1$, the range of $f$ must also contain some $v_k \le v_{\eta+1}$, from which it follows that $f(x) \le v_{\eta+1} + 1$ for all $x \in [0, 1]$. As such, we have that $f(x)$ must be within $[v_{\eta+1}-1, v_{\eta+1}+1]$ for any $x \in [0, 1]$, as promised.
\end{proof}

This fundamental result motivates our definition of the worst-case error $\optag_{p, \eta} (\mathcal{F}_q)$ in Section \ref{agnostic} to only count from the $2\eta+2^{\text{th}}$ round. 

We now move on to proving some bounds on $\optag_{p, \eta} (\mathcal{F}_q)$. We first establish the following fundamental result, characterizing the values of $\eta, p, q$, such that the worst-case error is finite. 

\begin{thm}
    For any integer $\eta \ge 1$, the value of $\optag_{p, \eta} (\mathcal{F}_q)$ is finite if and only if $p, q > 1$. 
\end{thm}

\begin{proof}
    Clearly, if $p=1$ or $q=1$, $\optag_{p, \eta} (\mathcal{F}_q) \ge \opt_p (\mathcal{F}_q) = \infty$. When $p, q > 1$, the key idea is to replicate the optimal strategy in the standard scenario for the same $p, q$, as if the adversary cannot lie, until something wrong occurs, indicating a lie—the total apparent error exceeds $\opt_p(\mathcal{F}_q)$ (which is finite by Theorem \ref{twopointtwelve}), or the learner gives feedback not in the allocated range $[v_{\eta+1}-1, v_{\eta+1}+1]$. Then, the learner forgets all previous feedback, starts all over again, and repeats. As this can happen at most $\eta+1$ times, and we can bound the error at every stage, the total error is bounded. 
\end{proof}

We now prove Theorem \ref{agnosticlowerbound}. Our proof is split into two parts: an upper bound through Lemma \ref{upper43}, and a lower bound through Lemma \ref{lower44}. The upper bound uses a similar but more sophisticated strategy as the previous result. 

\begin{lem}\label{upper43}
    For any $\eta \ge 1$, $p, q \ge 2$ we have $\optag_{p, \eta}(\mathcal F_q) \le 12\eta + 6$.
\end{lem}

\begin{proof}
    Split the domain $[0, 1]$ into four intervals: $I_1 = [0, \tfrac{1}{4})$, $I_2= [\tfrac{1}{4}, \tfrac{1}{2})$, $I_3=[\tfrac{1}{2}, \tfrac{3}{4})$, $I_4=[\tfrac{3}{4}, 1]$. We first claim that whenever the learner has adversary feedback on at least $2\eta+1$ points in any interval $I_j$, it can determine a constant $c$ such that for all $x \in I_j$, $f(x) \in [c-\tfrac{1}{2}, c+\tfrac{1}{2}]$. 

    To prove this, suppose that the learner receives the feedback set $S = \{(x_i, y_i): 1 \le i \le 2\eta+1\}$, where each $x_i \in I_j$ and $y_1 < \cdots < y_{2\eta+1}$. Indeed, letting $c= y_{\eta+1}$, as there must be a true point in $S$ with output at least $c$, and a true point in $S$ with output at most $c$, there must exist a $X_1 \in I_j$ with $f(X_1) = c$ by the Intermediate Value Theorem. 

    Now, if there exists any point $X_2 \in I_j$ with $f(X_2) > c + \frac{1}{2}$, then 
    \[J_q[f] > J_q\left[f_{\{(X_1, c), (X_2, c+\tfrac{1}{2})\}} \right] = \left|X_2 - X_1 \right| \left|\frac{\frac{1}{2}}{X_2 - X_1} \right|^q \ge \left(\frac{1}{4} \right)^{1-q} \left(\frac{1}{2} \right)^q \ge 1 \] as $|X_2-X_1| \le \tfrac{1}{4}$ and $q \ge 2$; contradiction. Similarly there cannot be a point $X_2 \in I_j$ with $f(X_2) < c - \frac{1}{2}$, so for all $X_2 \in I_j$, $f(X_2) \in [c-\frac{1}{2}, c+\frac{1}{2}]$, as claimed.

    Now, let $A$ be an optimal learning algorithm in the standard scenario for the same $p, q$, so that it always guarantees a total error value of at most $\opt_p(\mathcal{F}_q) = 1$. Consider the learning algorithm $\mathcal{A}$ trying to learn a function $f \in \mathcal{F}_q$ in the noisy scenario, following this strategy:
    
    Upon the end of the initial $2\eta+1$ rounds, first record the value $v$ such that $f$ can be bounded within $[v-1, v+1]$, previously shown to be possible. Now, let $\mathcal{A}$ enter Stage $1$. For any trial $t$ requesting the value of $f(x_t)$ for $x_t \in I_j$ for some $j$, if the learner knows at least $2\eta+1$ points already in $I_j$, predict exactly as algorithm $A$ would; otherwise, predict $v$. Furthermore, call the trials of the former kind \textit{mimicked trials}.

    Now, let Stage 1 continue until a \textit{mimicked trial}, with $x_t \in I_j$ for some $j$, where one of the following events happen: 1) The adversary reveals a point $(x_t, y)$ such that $y \notin [c-\frac{1}{2}, c+\frac{1}{2}]$, where $c$ is the constant such that $f$ is bounded within $[c-\frac{1}{2}, c+\frac{1}{2}]$ over $I_j$; 2) the algorithm $A$ predicts $\hat y$ with $\hat y \notin [c-\frac{1}{2}, c+\frac{1}{2}]$; or 3) the total error the learner perceives throughout all \textit{mimicked trials} of the Stage exceeds $\opt_p(\mathcal{F}_q)=1$.
    
    Once one of the three events happen, let $\mathcal{A}$ exit Stage 1, forget all previous adversary feedback, predict $v$ for the immediate next round, and enter Stage 2, repeating the same process until one of the three events happen again, whereupon it enters Stage 3, etc. Clearly, whenever one of the three events occur, a lie must have occurred since the end of the last stage. Thus, the learning process includes at most $\eta+1$ stages. 
    
    Now, onto error bounding. As there are $4$ intervals, each with at most $2\eta+1$ \textit{non-mimicked trials} which generate at most $1$ error, the total error from \textit{non-mimicked trials} is at most $4(2\eta+1) = 8\eta+4$.

    Now, onto bounding \textit{mimicked trials}. Note that every trial generates at most $1$ error. For each stage besides stage $\eta+1$, the total perceived error from the \textit{mimicked trials}, not counting the stage's last trial, is at most $1 + \opt_p(\mathcal{F}_q)=2$. The last trial of every stage generates an actual error of at most $1$.
    
    Over all stages besides stage $\eta+1$, there are at most $\eta$ \textit{mimicked trials} where the perceived error is not the actual error. The perceived error differs from the actual error by at most $1$ in these trials, since the correct value is in $[c-\frac{1}{2},c+\frac{1}{2}]$. Thus, the actual error from the first $\eta$ stages is at most $(2+1) \eta + \eta = 4\eta$.

    Finally, if the learning process reaches stage $\eta+1$, the adversary cannot lie any more, so the error from that stage is at most $1+\opt_p(\mathcal{F}_q)=2$. Therefore, in total, the sum of $p^{\text{th}}$ powers of actual errors for \textit{mimicked trials} over all stages is at most $4\eta+2$. Adding the error from \textit{non-mimicked trials}, we get at most $(4\eta + 2) + (8\eta+4) = 12\eta+6$.
\end{proof}

\begin{lem}\label{lower44}
    For any $\eta \ge 1$, $p, q \ge 2$, we have $\optag_{p, \eta} (\mathcal{F}_q) \ge 2\eta+1$.
\end{lem}

\begin{proof}
    Fix any algorithm $A$ for learning $\mathcal{F}_q$. Consider the following adversary strategy. For the first $2\eta+1$ rounds, repeatedly query the learner at $0$, and reveal the output to be $0$. Let $f(0)=0$, so that no lies have been used up yet. For the next $2\eta + 1$ rounds, the adversary repeatedly queries the learner at $1$. For the first $\eta$ of these rounds, reveal the value of $f(1)$ to be $-1$. Then, for the next $\eta$ rounds, reveal the value of $f(1)$ to be $1$. Then, on the next round, reveal the true value of $f(1)$ to be either $1$ or $-1$, whichever makes the total error function larger. Note that no matter what we choose, exactly $\eta$ lies have been used up. We claim that we can always guarantee a total error of at least $2\eta+1$. Indeed, note that \[\sum_{i=1}^{2\eta+1} |\hat y_i + 1|^p + \sum_{i=1}^{2\eta+1} |\hat y_i - 1|^p =\sum_{i=1}^{2\eta+1} (|\hat y_i + 1|^p + |\hat y_i - 1|^p) \ge 2(2\eta+1),\] where the last step comes from the fact that for any $p \ge 2$ and any $\hat y_i \in \mathbb{R}$, $|\hat y_i + 1|^p + |\hat y_i - 1|^p \ge 2$. To see why this is true, notice that if $\hat y_i \notin [-1, 1]$ the result is obvious; otherwise, Jensen's inequality gives the result. 
    
    This means at least one of $\sum_{i=1}^{2\eta+1} |\hat y_i + 1|^p$ and $\sum_{i=1}^{2\eta+1} |\hat y_i - 1|^p$ is at least $2\eta+1$. As the adversary can force an error of at least $2\eta+1$ regardless of the learner's predictions, $\optag_{p, \eta} (\mathcal{F}_q) \ge 2\eta+1$.
\end{proof}

Combining the bounds, we obtain a precise bound on $\optag_{p, \eta}(\mathcal{F}_q)$ for any $p, q \ge 2$ and $\eta \ge 1$. 

\begin{thm}
    For any $\eta \ge 1$, $p, q \ge 2$, we have that $\optag_{p, \eta}(\mathcal{F}_q) = \Theta(\eta)$. 
\end{thm}

\section{Discussion}\label{discussion}

With the results of this paper, we have completed the characterization of the ordered pairs $(p, q)$ for which $\opt_p(\mathcal{F}_q)$ is finite. This problem has been open since Kimber and Long \cite{kl} defined the model of online learning of smooth functions. For the standard learning scenario, a natural next step would be to establish lower and upper bounds on $\opt_p(\mathcal{F}_q)$ that match up to a constant factor for every $p, q \ge 1$. 

The paper \cite{geneson} established precisely that $\opt_p(\mathcal{F}_q) = 1$ for all $(p, q)$ with either $p, q \ge 2$ or $1 < q < 2$ and $p \ge 2 + \frac{1}{q-1}$. Another direction would be to identify more pairs of $(p, q)$ such that $\opt_p(\mathcal{F}_q) = 1$. 

Having studied the online learning of polynomials and established that polynomials in $\mathcal{F}_q$ are not any easier to learn than general functions in $\mathcal{F}_q$, it remains to investigate the worst-case error of learning other special subsets of $\mathcal{F}_q$. In particular, we invite future research on generalizing the ideas for our proof that $\opt_p(\mathcal{P}_q) = \opt_p(\mathcal{F}_q)$ to other special subsets $\mathcal{A}_q$ of $\mathcal{F}_q$. Indeed, the crux of our proof for $\mathcal{P}_q$ was a connection with the Weierstrass Approximation Theorem, which allowed us to use polynomials to uniformly approximate any continuous real-valued function. Yet, there exists a generalization of the Weierstrass Approximation Theorem, the Stone-Weierstrass Theorem. This theorem offers a condition on general subsets of the set of continuous functions over $[0, 1]$, which if satisfied implies that the functions in the subset can uniformly approximate every continuous real function over $[0, 1]$. As such, the Stone-Weierstrass Theorem may be a gateway to proving that more special subsets $\mathcal{A}_q$ of $\mathcal{F}_q$ are as hard to learn as $\mathcal{F}_q$ itself. Specifically, we conjecture that it is not easier to learn sums of exponential functions in $\mathcal{F}_q$, or trigonometric polynomials in $\mathcal{F}_q$, than $\mathcal{F}_q$ in general. 

In the noisy learning scenario, we have shown that $\optag_{p, \eta} (\mathcal{F}_q) = \Theta(\eta)$ for any $\eta$, and all $p, q \ge 2$. However, we have yet to obtain precise values of $\optag_{p, \eta} (\mathcal{F}_q)$ at any particular values of $p, q, \eta$. This could be an interesting research direction, alongside establishing upper bounds on $\optag_{p, \eta} (\mathcal{F}_q)$ when $p < 2$ or $q < 2$.

Motivated by our lower bound of $2\eta+1$ for all $p, q \ge 2$, we make the following conjecture about the behavior of the noisy learning scenario as the parameter $p$ approaches infinity.

\begin{conj}
    For any $\eta \ge 1$, $q \ge 1$, we have $\lim_{p \rightarrow \infty} \optag_{p, \eta} (\mathcal{F}_q) = 2\eta+1$.
\end{conj}

It would also be interesting to investigate online learning of smooth functions with other modes of adversary feedback. In particular, delayed ambiguous reinforcement has been studied for the online learning of classifiers \cite{Auer, Feng}. Another possibility would be feedback on the absolute value of prediction error. 

\section{Acknowledgements}
I would like to thank my PRIMES mentor, Dr.~Jesse Geneson, for introducing me to this research topic, providing new ideas for research directions, and offering useful feedback on my proofs. I am immensely grateful for his continued patience throughout the entirety of my work on this paper. I would also like to thank the MIT PRIMES-USA program for providing this wonderful research opportunity.


\begin{thebibliography}{}
\bibitem{angluin}D. Angluin, Queries and concept learning. \textit{Machine Learning} \textbf{2} (1988) 319--342.
\bibitem{Auer} P. Auer and P.M. Long. Structural results about on-line learning models with and without queries. \textit{Machine Learning} \textbf{36} (1999) 147-181.
\bibitem{barron}A. Barron, Approximation and estimation bounds for artificial neural networks. \textit{Workshop on Computational Learning Theory} (1991).
\bibitem{clw}N. Cesa-Bianchi, P.M. Long, and M.K. Warmuth, Worst-case quadratic loss bounds for prediction using linear functions and gradient descent. \textit{IEEE Transactions on Neural Networks} 7 (1996) 604--619.
\bibitem{cesa}N. Cesa-Bianchi, Y. Freund, D. Helmbold, and M. Warmuth. Online prediction and conversion strategies. \textit{Machine Learning}, 25:71–110, 1996.
\bibitem{FM}V. Faber and J. Mycielski, Applications of learning theorems. \textit{Fundamenta Informaticae} \textbf{15} (1991) 145--167.
\bibitem{Filmus}Y. Filmus, S. Hanneke, I. Mehalel, and S. Moran. Optimal prediction using expert advice and randomized littlestone dimension. \textit{Proceedings of Thirty Sixth Conference on Learning Theory}, pages 773–836, 2023.
\bibitem{Feng}R. Feng, J. Geneson, A. Lee, and E. Slettnes, Sharp bounds on the price of bandit feedback for several models of mistake-bounded online learning. \textit{Theoretical Computer Science} \textbf{965}: 113980 (2023).
\bibitem{geneson}J. Geneson and E. Zhou. Online learning of smooth functions. \textit{Theoretical Computer Science} \textbf{979C}: 114203 (2023).
\bibitem{hardle}W. Hardle, Smoothing techniques. Springer Verlag (1991).
\bibitem{kl}D. Kimber and P. M. Long, On-line learning of smooth functions of a single variable. \textit{Theoretical Computer Science} \textbf{148} (1995) 141--156.
\bibitem{littlestone}N. Littlestone, Learning quickly when irrelevant attributes abound: A new linear-threshold algorithm. \textit{Machine Learning} \textbf{2} (1988) 285--318.
\bibitem{lw}N. Littlestone and M.K. Warmuth, The weighted majority algorithm. \textit{Proceedings of the 30th Annual Symposium on the Foundations of Computer Science} (1989)
\bibitem{long}P. M. Long, Improved bounds about on-line learning of smooth functions of a single variable. \textit{Theoretical Computer Science} \textbf{241} (2000) 25--35.
\bibitem{mycielski}J. Mycielski, A learning algorithm for linear operators. \textit{Proceedings of the American Mathematical Society} \textbf{103} (1988) 547--550.
\end{thebibliography}
\end{document}